%% file: paper.tex
\documentclass{ecai}  

\usepackage{graphicx}
\usepackage{latexsym}

\ecaisubmission      

\paperid{2000}        

\usepackage{booktabs,multirow}

\usepackage{times}
\usepackage{soul}
\usepackage{url}
\usepackage[hidelinks]{hyperref}
\usepackage[utf8]{inputenc}
\usepackage[small]{caption}
\usepackage{subcaption}
\usepackage{graphicx}
\usepackage{amsmath}
\usepackage{amsthm}
\usepackage{mathtools,thmtools,thm-restate}
\usepackage{amsfonts,amssymb}
\usepackage{algorithm}
\usepackage[switch]{lineno}
\usepackage[noend]{algpseudocode}
\usepackage{newfloat}
\usepackage{listings}
\usepackage{comment}
\usepackage[dvipsnames]{xcolor}
\usepackage{array}
\usepackage[noabbrev,nameinlink,capitalize]{cleveref}
\usepackage{titlesec}
\usepackage{enumitem}
\usepackage{etoolbox}
\usepackage{xspace}
\usepackage{setspace}
\usepackage{relsize}
\usepackage{pifont}
\usepackage{stmaryrd}
\usepackage{booktabs,multirow,siunitx} 
\usepackage{nicefrac}

\usepackage{tikz}
\usepackage{pgf}
\usepackage{forest}

\usetikzlibrary{arrows,decorations.pathmorphing,decorations.footprints,fadings,calc,trees,mindmap,shadows,decorations.text,patterns,positioning,shapes,matrix,fit}
\usetikzlibrary{graphs}
\usetikzlibrary{decorations.pathreplacing}
\usetikzlibrary{arrows,shadows,backgrounds}
\usetikzlibrary{arrows.meta}
\usetikzlibrary{positioning,fit}
\usetikzlibrary{automata}
\usetikzlibrary{matrix}
\usetikzlibrary{karnaugh}
\usetikzlibrary{shapes.symbols,shapes.misc,shapes.arrows}
\usetikzlibrary{matrix,chains,scopes,decorations.pathmorphing}

\usetikzlibrary{fit,calc}

\newtheorem{example}{Example}

\input{defs}

\begin{document}

\begin{frontmatter}

\input{preamble}
\input{abs}

\end{frontmatter}

\input{intro}

\input{prelim}
\input{xpsi}
\input{xtra}

\input{res}

\input{relw}
\input{conc}

%
\iftrue
\input{acks}

\fi

\input{replbib}
\input{togbbl} 

\iftoggle{mkbbl}{
  \bibliography{refs,team,mf}
}{

\input{paper.bibl}
}

\clearpage
\input{appendix}


\end{document}

%% file: defs.tex
\input{cdefs}

\floatstyle{ruled}
\newfloat{listing}{tb}{lst}{}
\floatname{listing}{Listing}
\newtheorem{lemma}{Lemma}
\newtheorem{proposition}{Proposition}
\newtheorem{definition}{Definition}
\newtheorem{corollary}{Corollary}

\crefname{theorem}{Theorem}{Theorems}
\crefname{lemma}{Lemma}{Lemmas}
\crefname{proposition}{Proposition}{Propositions}
\crefname{definition}{Definition}{Definitions}
\crefname{corollary}{Corollary}{Corollaries}
\crefname{example}{Example}{Examples}
\crefname{claim}{Claim}{Claims}
\crefname{assumption}{Assumption}{Assumptions}

\newcommand{\cmark}{$\top$}
\newcommand{\xmark}{$\bot$}

\newcommand{\fml}[1]{{\mathcal{#1}}}

\newcommand{\tn}[1]{\textnormal{#1}}
\newcommand{\tbf}[1]{\textbf{#1}}
\newcommand{\mbf}[1]{\ensuremath\mathbf{#1}}
\newcommand{\msf}[1]{\ensuremath\mathsf{#1}}
\newcommand{\mbb}[1]{\ensuremath\mathbb{#1}}
\newcommand{\mrm}[1]{\ensuremath\mathrm{#1}}
\newcommand{\mfrak}[1]{\ensuremath\mathfrak{#1}}

\newcommand{\waxp}{\ensuremath\mathsf{WAXp}}
\newcommand{\wcxp}{\ensuremath\mathsf{WCXp}}
\newcommand{\axp}{\ensuremath\mathsf{AXp}}
\newcommand{\cxp}{\ensuremath\mathsf{CXp}}

\newcommand{\waxpui}{\ensuremath\mathsf{WAXp}}
\newcommand{\wpaxp}{\ensuremath\mathsf{WPAXp}}
\newcommand{\wpaxpui}{\ensuremath\mathsf{WPAXp}}

\DeclareMathOperator*{\consistent}{\textnormal{\bfseries{\textsf{CO}}}\!}

\newcommand{\onexp}{\ensuremath\mathsf{oneXP}}

\newcommand{\frp}{\ensuremath\mathsf{isRelevant}}

\newcommand{\vars}{\ensuremath\mathsf{Vars}}

\newcommand{\predicate}{\mbb{P}}
\newcommand{\predaxp}{\predicate_{\tn{axp}}}
\newcommand{\predcxp}{\predicate_{\tn{cxp}}}

\newcommand{\bigland}{\ensuremath\bigwedge}

\newcounter{tableeqn}[table]

\usepackage{booktabs,multirow}
\usepackage{amsmath}

\DeclareMathOperator*{\limply}{\rightarrow}

\definecolor{darkslategray}{rgb}{0.18, 0.31, 0.31} 
\definecolor{platinum}{rgb}{0.9, 0.89, 0.89} 
\definecolor{gray}{rgb}{.4,.4,.4}
\definecolor{midgrey}{rgb}{0.5,0.5,0.5}
\definecolor{middarkgrey}{rgb}{0.35,0.35,0.35}
\definecolor{darkgrey}{rgb}{0.3,0.3,0.3}
\definecolor{darkred}{rgb}{0.7,0.1,0.1}
\definecolor{midblue}{rgb}{0.2,0.2,0.7}
\definecolor{darkblue}{rgb}{0.1,0.1,0.5}
\definecolor{darkgreen}{rgb}{0.1,0.5,0.1}
\definecolor{defseagreen}{cmyk}{0.69,0,0.50,0}

\newcommand{\anote}[1]{\medskip\noindent$\llbracket$\textcolor{darkred}{antonio}: \emph{\textcolor{middarkgrey}{#1}}$\rrbracket$\medskip}
\newcommand{\rnote}[1]{\medskip\noindent$\llbracket$\textcolor{darkred}{ramon}: \emph{\textcolor{middarkgrey}{#1}}$\rrbracket$\medskip}

\newcommand{\jnoteF}[1]{}
\newcommand{\rnoteF}[1]{}

\newcolumntype{L}[1]{>{\raggedright\let\newline\\\arraybackslash\hspace{0pt}}m{#1}}
\newcolumntype{C}[1]{>{\centering\let\newline\\\arraybackslash\hspace{0pt}}m{#1}}
\newcolumntype{R}[1]{>{\raggedleft\let\newline\\\arraybackslash\hspace{0pt}}m{#1}}

\makeatletter
\def\xscriptsize{\@setfontsize\scriptsize{8pt}{9pt}}
\def\xtiny{\@setfontsize\tiny{7pt}{8pt}}
\def\xlarge{\@setfontsize\large{11pt}{12pt}}
\def\xLarge{\@setfontsize\Large{12pt}{14pt}}
\def\xLARGE{\@setfontsize\LARGE{13pt}{15pt}}
\def\xhuge{\@setfontsize\huge{20pt}{20pt}}
\def\xHuge{\@setfontsize\Huge{28pt}{28pt}}
\makeatother

\tikzset{
  0 my edge/.style={densely dashed, my edge},
  my edge/.style={-{Stealth[]}},
}

\def\HiLi{\leavevmode\rlap{\hbox to \linewidth{\color{platinum}\leaders\hrule height .8\baselineskip depth .5ex\hfill}}}

\setitemize{noitemsep,topsep=0.5pt,parsep=0.5pt,partopsep=0.5pt}
\setenumerate{noitemsep,topsep=0.5pt,parsep=0.5pt,partopsep=0.5pt}
\setlist{nosep,leftmargin=0.45cm}

\providetoggle{long}
\settoggle{long}{false}

\makeatletter
\algnewcommand{\LineComment}[1]{\Statex \hskip\ALG@thistlm \(\triangleright\) #1}
\makeatother

\newcommand{\SAT}{\ensuremath\mathsf{SAT}}
\newcommand{\outc}{\ensuremath\mathsf{outc}}
\newcommand{\True}{\textbf{true}}
\newcommand{\False}{\textbf{false}}

\newcommand{\prtaxp}{\ensuremath\mathsf{reportAXp}}
\newcommand{\prtcxp}{\ensuremath\mathsf{reportCXp}}

\titlespacing*{\paragraph}{0pt}{1.25ex plus 1ex minus .2ex}{1em}


%% file: cdefs.tex
\definecolor{gray}{rgb}{.4,.4,.4}
\definecolor{midgrey}{rgb}{0.5,0.5,0.5}
\definecolor{middarkgrey}{rgb}{0.35,0.35,0.35}
\definecolor{darkgrey}{rgb}{0.3,0.3,0.3}
\definecolor{darkred}{rgb}{0.7,0.1,0.1}
\definecolor{midblue}{rgb}{0.2,0.2,0.7}
\definecolor{darkblue}{rgb}{0.1,0.1,0.5}
\definecolor{defseagreen}{cmyk}{0.69,0,0.50,0}

\definecolor{tyellow1}{HTML}{FCE94F}
\definecolor{tyellow2}{HTML}{EDD400}
\definecolor{tyellow3}{HTML}{C4A000}
\definecolor{torange1}{HTML}{FCAF3E}
\definecolor{torange2}{HTML}{F57900}
\definecolor{torange3}{HTML}{C35C00}
\definecolor{tbrown1}{HTML}{E9B96E}
\definecolor{tbrown2}{HTML}{C17D11}
\definecolor{tbrown3}{HTML}{8F5902}
\definecolor{tgreen1}{HTML}{8AE234}
\definecolor{tgreen2}{HTML}{73D216}
\definecolor{tgreen3}{HTML}{4E9A06}
\definecolor{tblue1}{HTML}{729FCF}
\definecolor{tblue2}{HTML}{3465A4}
\definecolor{tblue3}{HTML}{204A87}
\definecolor{tpurple1}{HTML}{AD7FA8}
\definecolor{tpurple2}{HTML}{75507B}
\definecolor{tpurple3}{HTML}{5C3566}
\definecolor{tred1}{HTML}{EF2929}
\definecolor{tred2}{HTML}{CC0000}
\definecolor{tred3}{HTML}{A40000}
\definecolor{tlgray1}{HTML}{EEEEEC}
\definecolor{tlgray2}{HTML}{D3D7CF}
\definecolor{tlgray3}{HTML}{BABDB6}
\definecolor{tdgray1}{HTML}{888A85}
\definecolor{tdgray2}{HTML}{555753}
\definecolor{tdgray3}{HTML}{2E3436}

\newcommand{\hlight}[1]{{\color{darkred}#1}}
\newcommand{\rhlight}[1]{\hlight{#1}}
\newcommand{\dghlight}[1]{{\color[RGB]{0,120,0}#1}}

%% file: preamble.tex
\iftrue
\makeatletter
\let\orcids@fmt\relax
\makeatother
\fi




\title{On Logic-Based Explainability\\with Partially Specified Inputs}

\iftrue
\author[A]{Ramón Béjar. Email: ramon@udl.cat.}
\author[A]{António Morgado. Email: ajrmorgado@gmail.com.}
\author[A]{Jordi Planes. Email: jordi.planes@udl.cat.}
\author[B]{Joao Marques-Silva. Email: joao.marques-silva@irit.fr.}

\address[A]{Univ.~Lleida, Leida, Spain}
\address[B]{IRIT, CNRS, Toulouse, France}
\fi

%% file: abs.tex
\begin{abstract}
  In the practical deployment of machine learning (ML) models, missing
  data represents a recurring challenge.
  Missing data is often addressed when training ML models.
  But missing data also needs to be addressed when deciding predictions and when explaining those predictions.
  Missing data represents an opportunity to partially specify the inputs of the prediction to be explained.
  This paper studies the computation of logic-based explanations in
  the presence of partially specified inputs.
  The paper shows that most of the algorithms proposed in recent years
  for computing logic-based explanations can be generalized for
  computing explanations given the partially specified inputs.
  One related result is that the complexity of computing logic-based 
  explanations remains unchanged. 
  A similar result is proved in the case of logic-based explainability subject
  to input constraints.
  Furthermore, the proposed solution for computing explanations given partially specified inputs is applied to classifiers obtained from well-known
  public datasets, thereby illustrating a number of novel
  explainability use cases.
  %
\end{abstract}

%% file: intro.tex
\section{Introduction} \label{sec:intro}

In practical uses of machine learning (ML) models, there are often
situations of missing data~\cite{rubin-bk19}, i.e.\ feature values
that are either unknown or are preferred to be left unspecified.
%
%
%
Datasets may contain missing feature values, and so the
learning of ML models must account for missing data~\cite{shamir-icml08,graham-arp09,shamir-ml10,graham-bk12,peng-sp13,little-jpp14,little-ps18,rubin-bk19,morvan2021}.
%
%
%
For example,
the \emph{Bosch Production Line Performance}\footnote{\url{https://www.kaggle.com/competitions/bosch-production-line-performance}}
dataset has about 80\% of missing values~\cite{Mangal16},
%
the \emph{Pima Indians Diabetes}\footnote{\url{https://www.kaggle.com/datasets/uciml/pima-indians-diabetes-database}}
dataset is known to have about 60\% of features with missing values,
%
%
and the \emph{Water Potability}\footnote{\url{https://www.kaggle.com/datasets/adityakadiwal/water-potability}}
dataset has about 20\% missing values in the feature ''Sulfate''.
%
%

However, in some situations the cost and/or the effort of assigning
values to some features (\cite{saar-tsechankty07}) offers strong motivation for deciding
predictions in the presence of missing feature values. 
For example, features that are (far) more time-consuming to acquire,
or (much) more expensive, or even that are invasive for patients, are 
in general preferred to be considered only if truly necessary for a
prediction.
Concrete practical use cases include medical
diagnosis~\cite{bro-cils05,manski13-pnas13,disdier-bmcmidm16,demoor-cmpb22},
but also in situations where costly/expensive field measures are
preferred not to be tried unless necessary~\cite{zeng-wiseml22,cesa-bianchi11}.
Another situation relates to how sensitive is the information of the feature's value.
For example a user may prefer not to disclose information about its own
family income~\cite{schenker06}.
Besides model learning and model prediction, missing feature values
are also a concern for model
explainability~\cite{das-nips18,vandenbroeck-ijcai19,vandenbroeck-corr20,das-corr20}.

%
%

In Deep Learning, ablation studies~\cite{Sheikholeslami21} and 
pruning~\cite{Blalock20} are techniques which ablate or shrink 
the network. 
In both cases, explanations with missing values may give 
important information of how the removal of parts of the
network (i.e. the missing values) could affect to its consistency~\cite{morvan2021,Ayme22}
fairness~\cite{Jeong22}, or even increase its accuracy~\cite{Yeom21}.

For the problems of model prediction and model explainability, even when
the ML model is known, model-based reasoning is complicated by the
existence of unspecified features in the input instance, since these may
take any value from their domain.
As a result, the existence of missing feature values in the inputs
raises questions about what can be predicted, but also
about what needs to be explained.
%
%
%
%
%

This paper formalizes the computation of predictions and the
identification of explanations when some features are left
unspecified. Although the inputs to the ML model can be left
unspecified, the model itself is assumed to be known and
it can be reasoned about.
%

Regarding the set of specified features, our goal is to understand
which class (or set of classes) can be predicted given the unspecified
features.
This requires generalizing the definition of classification function
by allowing sets of classes to be predicted, as opposed to single
class prediction considered in earlier work on explainability.
Under the assumption that the set of predicted classes is deemed of
interest, we investigate how rigorous logic-based explanations can be
computed.
In contrast with recent
work~\cite{das-nips18,vandenbroeck-ijcai19,vandenbroeck-corr20}, we
prove that the so-called prime implicant (PI)
explanations~\cite{darwiche-ijcai18} can be computed with minor 
changes to existing algorithms, and such that the complexity of
computing a single explanation remains unchanged with respect to
\emph{each} family of classifiers. For example, polynomial-time
tractability results, obtained for completely specified inputs,
generalize to the case of incompletely specified inputs. Furthermore,
the complexity of enumerating explanations remains unchanged.

Besides computing one explanation, we investigate how enumeration
of explanations can be instrumented, and
how to consider unspecified inputs when querying for the relevance
of a feature.
We also consider how to compute a probabilistic explanation and how constraints on the
inputs can be taken into account in the presence of unspecified values in the input.

The paper is organized as follows. \cref{sec:prelim} introduces the
definitions and notation used in the rest of the paper, including 
basic definitions on logic-based explainability.
Afterwards, \cref{sec:xpsi} develops modifications to the definitions
of classifiers and explanations which are required to account for
unspecified inputs.
Furthermore, changes to the basic explainability algorithms are also
described.
\cref{sec:xtra} studies a number of additional properties, and a set of generalizations of explanations.
%
\cref{sec:res} summarizes the experiments, and aims to illustrate the
generality of the ideas proposed in the paper.
\cref{sec:relw} briefly overviews recent related work. 
Finally, \cref{sec:conc} summarizes the paper's contributions and
identifies directions of research.

%% file: prelim.tex
\section{Preliminaries} \label{sec:prelim}

\paragraph{Basics.}
%
The paper assumes basic knowledge of computational complexity, namely
the classes of decision problems P and NP~\cite{arora-bk09}, and their
function variants.
The paper also assumes basic knowledge of propositional logic,
including the Boolean satisfiability (SAT) decision problem for
propositional logic formulas in conjunctive normal form (CNF), 
and the use of SAT solvers as oracles. 
The interested
reader is referred to textbooks on these
topics~\cite{arora-bk09,sat-hbk21}.

\paragraph{Classification problems in ML.}
Throughout the paper, we will consider classifiers as the underlying
ML model. 
Classification problems are defined on a set of features (or
attributes) $\fml{F}=\{1,\ldots,m\}$ and a set of classes
$\fml{K}=\{c_1,c_2,\ldots,c_K\}$.
Each feature $i\in\fml{F}$ takes values from a domain
$\mbb{D}_i$. Domains of features are aggregated as a sequence
$\mbb{D}=\langle\mbb{D}_1,\ldots,\mbb{D}_m\rangle$. Domains are
categorical or ordinal, and each domain can be defined on boolean,
integer/discrete or real values. 
Feature space $\mbb{F}$ denotes the cartesian product of the domains,
i.e.\  $\mbb{F}=\mbb{D}_1\times{\mbb{D}_2}\times\ldots\times{\mbb{D}_m}$.
%
The notation $\mbf{x}=(x_1,\ldots,x_m)$ denotes an arbitrary point in
feature space, where each $x_i$ is a variable taking values from
$\mbb{D}_i$. The set of variables associated with the features is
$X=\{x_1,\ldots,x_m\}$.
Moreover, the notation $\mbf{v}=(v_1,\ldots,v_m)$ represents a
specific point in feature space, where each $v_i$ is a constant
value from $\mbb{D}_i$. 
A classifier $\fml{M}$ is characterized by a (non-constant)
\emph{classification function} $\kappa$ that maps feature space
$\mbb{F}$ into the set of classes $\fml{K}$,
i.e.\ $\kappa:\mbb{F}\to\fml{K}$.
We also assume that $\kappa$ is surjective, i.e.\ for any class in
$c\in\fml{K}$ there is at least one point $\mbf{x}\in\mbb{F}$ such that
$\kappa(\mbf{x})=c$.
A classification problem is a tuple
$\fml{C}=(\fml{F},\mbb{D},\mbb{F},\fml{K},\kappa)$.
An \emph{instance} 
is a pair $(\mbf{v}, c)$, where $\mbf{v}\in\mbb{F}$,
$c\in\fml{K}$, and $c=\kappa(\mbf{v})$.
%

\paragraph{Logic-based explanations.}
%
Prime implicant (PI) explanations~\cite{darwiche-ijcai18} represent a
minimal set of literals (relating a feature value $x_i$ and a constant 
$v_i\in\mbb{D}_i$) that are logically sufficient for the prediction.
PI-explanations are related with logic-based abduction, and so are
also referred to as \emph{abductive explanations}
(AXp's)~\cite{inms-aaai19}. 
There is comprehensive evidence that AXp's offer guarantees of rigor
that are not offered by other alternative explanation
approaches~\cite{ignatiev-ijcai20}. 
%
More recently, AXp's have been studied
in terms of their computational
complexity~\cite{barcelo-nips20,marquis-kr21}.
Furthermore, there is a growing body of recent work on logic-based
explanations~\cite{darwiche-jair21,kutyniok-jair21,ims-ijcai21,kwiatkowska-ijcai21,msgcin-icml21,barcelo-nips21,mazure-cikm21,rubin-aaai22,msi-aaai22,amgoud-ijcai22,barcelo-nips22}. 

Given a classification problem $\fml{C}$ and instance $(\mbf{v},c)$,
with $\mbf{v}\in\mbb{F}$, $c\in\fml{K}$ and $c=\kappa(\mbf{v})$, 
an explanation problem is a tuple $\fml{E}=(\fml{C},(\mbf{v},c))$.
%
%
For an explanation problem $\fml{E}=(\fml{C},(\mbf{v},c))$,
an AXp is any subset-minimal set $\fml{X}\subseteq\fml{F}$ such that,
the following predicate holds:
\begin{align} \label{eq:axp}
  \waxp(\fml{X}) & := 
   \forall(\mbf{x}\in\mbb{F}).
  \left[
    \bigland\nolimits_{i\in{\fml{X}}}(x_i=v_i)
    \right]
  \limply(\kappa(\mbf{x})=c)
\end{align}
If a set $\fml{X}\subseteq\fml{F}$ is not 
minimal but \eqref{eq:axp} holds, then $\fml{X}$ is referred to as a
\emph{weak} AXp.
%
%
%
Clearly, the predicate $\waxp$ maps $2^{\fml{F}}$ into $\{\bot,\top\}$
(or $\{\False,\True\}$).
Given $\mbf{v}\in\mbb{F}$, an AXp $\fml{X}$ represents an irreducible
(or minimal) subset of the features which, if assigned the values
dictated by $\mbf{v}$, are sufficient for the prediction $c$,
i.e.\ value changes to the features not in $\fml{X}$ will not change
the prediction. (Alternatively, an AXp can be viewed as an answer to a
``Why?'' (the prediction) question.)
We can use the definition of the predicate $\waxp$ to formalize the
definition of the predicate $\axp$, also defined on subsets $\fml{X}$
of $\fml{F}$:
\begin{equation} \label{eq:axp2}
  \axp(\fml{X}) :=
  \waxp(\fml{X}) \land
  \forall(\fml{X}'\subsetneq\fml{X}).
  \neg\waxp(\fml{X}')
\end{equation}

Complementary to AXp's, contrastive explanations~\cite{inams-aiia20}
(CXp's) answer a ``Why Not?'' question, and identify subset-minimal
sets of features which, if allowed to take any value, are sufficient
to change the prediction\footnote{%
From a logical perspective, there can exist differences between
contrastive and counterfactual explanations~\cite{lorini-corr21}. For
the purposes of this paper, such difference is non-existing, and we
will just use the term contrastive explanation.}:
\begin{align} \label{eq:cxp}
  \wcxp(\fml{Y}) & := 
   \exists(\mbf{x}\in\mbb{F}).
  \left[
    \bigland\nolimits_{i\not\in{\fml{Y}}}(x_i=v_i)
    \right]
  \land(\kappa(\mbf{x})\not=c)
\end{align}
Accordingly, we can define the predicate $\cxp$:
\begin{equation} \label{eq:cxp2}
  \cxp(\fml{Y}) :=
  \wcxp(\fml{Y}) \land
  \forall(\fml{Y}'\subsetneq\fml{Y}).
  \neg\wcxp(\fml{Y}')
\end{equation}

The definitions of $\waxp$ and $\wcxp$ ensure that these predicates
are \emph{monotone} and up-closed.
Indeed, if $\fml{X}\subseteq\fml{X}'\subseteq\fml{F}$, and if
$\fml{X}$ is a weak AXp, then $\fml{X}'$ is also a weak AXp, as the
fixing of more features will not change the prediction.
Similarly, for weak CXp's $\fml{Y}\subseteq\fml{Y}'\subseteq\fml{F}$,
if $\wcxp(\fml{Y})$, then $\wcxp(\fml{Y}')$.
%
%
%
Given the monotonicity of predicates $\waxp$/$\wcxp$, the definition
of predicates $\axp$/$\cxp$ can be simplified as follows, with
$\fml{X}\subseteq\fml{F}$ (resp.~$\fml{Y}\subseteq\fml{F}$):
%
%
\begin{align}
  \axp(\fml{X}) & ~:=~
  \waxp(\fml{X})
  \land_{t\in\fml{X}}\neg\waxp(\fml{X}\setminus\{t\}) \label{eq:axp3}
  \\
  \cxp(\fml{Y}) & ~:=~
  \wcxp(\fml{Y})
  \land_{t\in\fml{Y}}\neg\wcxp(\fml{Y}\setminus\{t\})  \label{eq:cxp3}
\end{align}

These simpler but equivalent definitions of AXp's/CXp's have important
practical significance, in that only a linear number of subsets needs
to be checked for, as opposed to exponentially many subsets
in~\eqref{eq:axp2}/\eqref{eq:cxp2}. As a result, the algorithms that
compute one AXp/CXp are based
on~\eqref{eq:axp3}/\eqref{eq:cxp3}~\cite{msi-aaai22}.
%

\input{./texfigs/runex}

Given an explanation problem $\fml{E}=(\fml{C},(\mbf{v},c))$,
the sets of AXp's and CXp's are defined as follows,
\begin{align}
  \mbb{A}(\fml{E}) & = \{\fml{X}\subseteq\fml{F}\,|\,\axp(X)\}\\
  \mbb{C}(\fml{E}) & = \{\fml{Y}\subseteq\fml{F}\,|\,\cxp(Y)\}
\end{align}
Clearly, the sets above are parameterized on the associated
explanation problem.
Moreover, AXp's and CXp's respect an instrumental duality
property~\cite{inams-aiia20}, which builds on Reiter's seminal work on
model-based diagnosis~\cite{reiter-aij87}.

\begin{proposition}[Duality] \label{prop:dual}
  For an explanation problem $\fml{E}$, each AXp in $\mbb{A}(\fml{E})$
  is a minimal hitting set (MHS) of the CXp's in $\mbb{C}(\fml{E})$
  and vice-versa.
\end{proposition}

\paragraph{Running example.}
Although the results in this paper apply to \emph{any} ML classifier,
we will illustrate the main ideas with an intuitive family of
classifiers, namely decision trees (DTs). Concretely, the paper considers
a decision tree for a medical diagnosis example adapted
from~\cite[Fig.~1]{tanner-plosntd08}, which targets the prediction of
dengue fever.

\begin{example} \label{ex:runex01}
  The example decision tree is shown in~\cref{fig:runex}.
  The set of features is $\fml{F}=\{1,2,3,4,5,6\}$, each of which with
  a real-valued domain. The set of classes is
  $\fml{K}=\{\ominus\ominus,\ominus,\oplus,\oplus\oplus\}$.
  Each class has the following interpretation:
  i) $\oplus\oplus$ indicates that dengue fever is probable (or very likely);
  ii) $\oplus$ that dengue fever is likely;
  iii) $\ominus$ that non-dengue fever is likely;
  and
  iv) $\ominus\ominus$ that non-dengue fever is probable\footnote{%
  Although orthogonal to the example, the features' original names are
  as follows (from~\cite{tanner-plosntd08}):
  HCT: hematocrit,
  Lymp.: total number of lymphocytes,
  Neutr.: total number of neutrophils, 
  PLT: platelet count,
  Temp.: body temperature,
  WBC=white blood cell count.}.
  %
\end{example}

%% file: texfigs/runex.tex
\begin{figure*}[t]
  \captionsetup[subfigure]{aboveskip=-1pt,belowskip=-1pt}
  \begin{minipage}{0.6\textwidth}
    \begin{subfigure}{\textwidth}
      \scalebox{0.925}{\input{./texfigs/tree01}}
    \end{subfigure}

    \bigskip

    \begin{subfigure}{\textwidth}
      \begin{center}
        \renewcommand{\arraystretch}{1.0125}
        \begin{tabular}{C{5cm}} \toprule
          \multicolumn{1}{c}{Incompletely specified input}\\ \toprule
          $\mbf{z}=(35.0,0.25,\mfrak{u},100.0,\mfrak{u},5.0)$ \\ \midrule
          $\kappa'(\mbf{z})=\{\oplus,\oplus\oplus\}$ \\
          \bottomrule
        \end{tabular}
      \end{center}
    \end{subfigure}
  \end{minipage}
  \begin{minipage}{0.4\textwidth}
    \begin{subfigure}{\textwidth}
      \begin{center}
        \renewcommand{\arraystretch}{1.0125}
        \begin{tabular}{C{1cm}C{1cm}C{2.25cm}} \toprule
          \multicolumn{3}{c}{Features \& predicates} \\ \toprule
          Feat.~\# & Pred. & Def. \\ \toprule 
          1 & $H_1$ & $\tn{HCT}\le41.2$ \\ \midrule
          2 & $L_1$ & $\tn{Lymp.}\le0.58$ \\ \midrule 
          3 & $N_1$ & $\tn{Neutr.}\le4.9$ \\ \midrule 
          4 & $P_1$ & $\tn{PLT}\le193$ \\ \midrule 
          4 & $P_2$ & $\tn{PLT}\le143$ \\ \midrule 
          5 & $T_1$ & $\tn{Temp.}>37.4$ \\ \midrule
          6 & $W_1$ & $\tn{WBC}\le6.0$ \\
          \bottomrule
        \end{tabular}
      \end{center}
    \end{subfigure}

    \bigskip\smallskip

    \begin{subfigure}{\textwidth}
      \begin{center}
        \renewcommand{\arraystretch}{1.0125}
        \begin{tabular}{C{5cm}} \toprule
          \multicolumn{1}{c}{Domains \& feature space} \\ \toprule
          $\mbb{D}_1=\cdots=\mbb{D}_6=\mbb{R}$ \\ \midrule
          $\mbb{F}=\mbb{D}_1\times\cdots\times\mbb{D}_6$ \\ \midrule
          $\mbb{D}_i'= \mbb{D}_i\cup\{\mfrak{u}\}, i=1,\ldots,6$ \\ \midrule
          $\mbb{F}'=\mbb{D}_1'\times\cdots\times\mbb{D}_6'$ \\
          \bottomrule
        \end{tabular}
      \end{center}
    \end{subfigure}
  \end{minipage}
  \caption{Example DT from the domain of medical diagnosis. Each internal node is associated with a predicate defined on one of the features.} \label{fig:runex}
\end{figure*}

%% file: texfigs/tree01.tex
%
\forestset{
  BDT/.style={
    for tree={
      l=1.5cm,s sep=1.0cm,
      if n children=0{}{circle}, 
      draw=black,
      text=black,
      edge={
        my edge
      },
      if n=1{}{
        edge+={0 my edge},
      },
    }
  },
}
\begin{forest}
  BDT
  [{$P_1$}, label={[yshift=-7.25ex]{{\xtiny1}}} 
    [{$W_1$}, label={[yshift=-7.75ex]{{\xtiny2}}}, 
    edge=thick,
      edge label={node[midway,left,xshift=-3.5pt] {{\xscriptsize$=1$}}}
      [{$T_1$}, label={[yshift=-7.125ex]{{\xtiny4}}}, 
        edge=thick,
        edge label={node[midway,left,xshift=-1.0pt] {{\xscriptsize$=1$}}}
        [\rhlight{$\oplus\oplus$}, label={[yshift=-5.125ex]{{\xtiny8}}},
          edge=thick,
          edge label={node[midway,left,xshift=-0.5pt] {{\xscriptsize$=1$}}},
          rectangle, fill={tred3!20}
        ]
        [{$P_2$}, label={[yshift=-7.25ex]{{\xtiny9}}}, 
          edge=thick,
          edge label={node[midway,right,xshift=0.5pt]
            {{\xscriptsize$=0$}}}
          [\rhlight{$\oplus$}, label={[yshift=-5.125ex]{{\xtiny14}}},
            edge=thick,
            edge label={node[midway,left,xshift=-0.5pt] {{\xscriptsize$=1$}}},
            rectangle, fill={tred3!20}
          ]
          [\dghlight{$\ominus$}, label={[yshift=-5.125ex]{{\xtiny15}}},
            edge label={node[midway,right,xshift=0.5pt] {{\xscriptsize$=0$}}},
            rectangle, fill={tgreen3!25}
          ]
        ]
      ]
      [{$H_1$}, label={[yshift=-7.5ex]{{\xtiny5}}}, 
        edge label={node[midway,right,xshift=0.5pt] {{\xscriptsize$=0$}}}
        [\rhlight{$\oplus$}, label={[yshift=-5.125ex]{{\xtiny10}}},
          edge label={node[midway,left,xshift=-0.5pt] {{\xscriptsize$=1$}}},
          rectangle, fill={tred3!20}
        ]
        [\dghlight{$\ominus$}, label={[yshift=-5.125ex]{{\xtiny11}}},
          edge label={node[midway,right,xshift=0.5pt] {{\xscriptsize$=0$}}},
          rectangle, fill={tgreen3!25}
        ]
      ]
    ]
    [{$L_1$}, label={[yshift=-7.25ex]{{\xtiny3}}}, 
      edge label={node[midway,right,xshift=2.25pt] {{\xscriptsize$=0$}}}
      [{$N_1$}, label={[yshift=-7.525ex]{{\xtiny6}}}, 
        edge label={node[midway,left,xshift=-0.5pt]
          {{\xscriptsize$=1$}}}
        [\rhlight{$\oplus$}, label={[yshift=-5.125ex]{{\xtiny12}}},
          edge label={node[midway,left,xshift=-1.5pt] {{\xscriptsize$=1$}}},
          rectangle, fill={tred3!20}
        ]
        [\dghlight{$\ominus$}, label={[yshift=-5.125ex]{{\xtiny13}}},
          edge label={node[midway,right,xshift=-1.5pt] {{\xscriptsize$=0$}}},
          rectangle, fill={tgreen3!25}
        ]
      ]
      [\dghlight{$\ominus\ominus$}, label={[yshift=-5.125ex]{{\xtiny7}}},
        edge label={node[midway,right,xshift=0.5pt] {{\xscriptsize$=0$}}},
        rectangle, fill={tgreen3!25}
      ]
    ]
  ]
\end{forest}

%% file: xpsi.tex
\section{Explanations with Partially Specified Inputs} \label{sec:xpsi}


\subsection{Partially-Specified Inputs}
%
We consider the situation of partially specified inputs to an ML
model, i.e.\ not all features of the input instance $\mbf{v}$ are assigned values, in which case we
say that there exist \emph{missing feature values} (or simply missing inputs).
To account for these missing values, we extend the definition of a
classifier to allow for \emph{partially-specified inputs}, as
follows.
Each domain $\mbb{D}_i$ is extended with a distinguished value
$\mfrak{u}$, not in any $\mbb{D}_i$, which indicates that feature $i$
is left unspecified, i.e.\ it could possibly take \emph{any} value
from its domain. Concretely, $\mbb{D}'_i=\mbb{D}_i\cup\{\mfrak{u}\}$.
Feature space (with unspecified inputs) becomes
$\mbb{F}'=\mbb{D}'_1\times\cdots\times\mbb{D}'_m$. 
A completely specified point in feature space $\mbf{v}\in\mbb{F}$ is
\emph{covered} by $\mbf{z}\in\mbb{F}'$, written
$\mbf{v}\sqsubseteq\mbf{z}$,
if
$\forall(i\in\fml{F}).\left[(v_i=z_i)\lor({z_i}=\mfrak{u})\right]$.
Given $\mbf{z}\in\mbb{F}'$, $\fml{F}$ is partitioned into $\fml{U}$
and $\fml{S}$, containing respectively the unspecified ($\fml{U}$) and
specified($\fml{S}$) features.

It will also be convenient to generalize the classification function
to map points in $\mbb{F}'$ into a set $\fml{T}\subseteq\fml{K}$\footnote{The case
where $\fml{T}=\fml{K}$ is uninteresting for our purposes and so it is
not considered.} of (target) classes, i.e.\ $\kappa':\mbb{F}'\to2^{\fml{K}}$.
Thus, given a partially specified input $\mbf{z}\in\mbb{F}'$, the
(generalized) prediction is some $\fml{T}\subsetneq\fml{K}$.
%
Hence, $\kappa'$ is defined as follows,
\begin{equation} \label{eq:defkp}
  \kappa'(\mbf{z})=\{c\in\fml{K}\,|\,\exists(\mbf{v}\in\mbb{F}).\mbf{v}\sqsubseteq\mbf{z}\land{c}=\kappa(\mbf{v})\}
\end{equation}
In general, we denote $\kappa'(\mbf{z})$ by $\fml{T}\subsetneq\fml{K}$,
and refer to $\fml{T}$ as the (generalized) prediction given
$\mbf{z}$.
Finally, a generalized (i.e.\ with partially specified inputs)
explanation problem is defined as a tuple
$\fml{E}'=(\fml{C},(\mbf{z},\fml{T}))$. 

\begin{example} \label{ex:runex02}
  For the running example (see~\cref{fig:runex}), we consider that
  some of the features are unspecified, namely features 3 (for Neutr.)
  and 5 (for Temp.). The other features are assigned the values shown
  in~\cref{fig:runex}.
  %
  Thus $\mbf{z}=(35.0,0.25,\mfrak{u},100.0,\mfrak{u},5.0)$,
  $\fml{F}$ is partitioned into $\fml{U}=\{3,5\}$, and
  $\fml{S}=\{1,2,4,6\}$,
  %
  and $\kappa'(\mbf{z})=\{\oplus,\oplus\oplus\}=\fml{T}$.
  %
  Moreover, we assume that the user considers the set of of possible
  predictions $\fml{T}=\{\oplus,\oplus\oplus\}$ to be
  satisfactory, and
  seeks an explanation for $\fml{T}$.
\end{example}

\begin{example}
  Considering again the running example, let $H_1=P_1=P_2=1$, $T_1=0$
  and $W_1$, $L_1$ and $N_1$ be left unspecified. By inspection we can
  conclude that $\fml{T}=\{\oplus\}$, independently of the values taken
  by $W_1$, $L_1$ and $N_1$. 
\end{example}

\subsection{Selecting $\fml{T}$}
%
Given $\mbf{z}\in\mbb{F}'$, what should the value of $\fml{T}$ be?
One possible solution is for $\fml{T}$ to be user-specified, e.g.\ it
represents a set of classes (or a class) of interest to the user.
Another solution is to use the definition above (see~\eqref{eq:defkp})
for computing $\kappa'(\mbf{z})$. One obstacle to exploiting this
solution in practice is that the number of (completely specified)
points to analyze is worst-case unbounded on the size of
$\fml{U}$~\footnote{%
Observe that, for features with discrete finite domains, the
worst-case would be exponential. However, we also account for
real-valued features. In such a situation, the worst-case number of 
points in feature space is unbounded.}.
%
In practice, a simpler solution is to construct $\fml{T}$, as follows.
\begin{enumerate}[nosep]
\item Set $\fml{T}=\emptyset$;
\item For each $c\in\fml{K}$, decide whether the following statement
  is consistent:
  \begin{equation}
    \exists(\mbf{x}\in\mbb{F}).\left(\land_{i\in\fml{S}}(x_i=z_i)\right)\land(c=\kappa(\mbf{x}))
  \end{equation}
\item If the statement above is consistent, then add $c$ to
  $\fml{T}$.
\end{enumerate}
For example, for families of classifiers for which computing one
explanation is in
P~\cite{darwiche-ijcai18,msgcin-icml21,iims-jair22},
then the proposed solution computes $\fml{T}$ in polynomial time, even
if the number of points in feature space covered by $\mbf{z}$ is
unbounded 
in the worst case.

The flexibility in the choice of $\fml{T}$, i.e.\ either inferred from
$\mbf{z}$ or specified by the user, is illustrated in the following
example.

\begin{example} \label{ex:runex03}
  For the DT in~\cref{fig:runex}, the user could specify
  $\mbf{z}_3=(35.0, 0.25, \mfrak{u}, 100.0, 36.3, 5.0)$, with
  $\kappa'(\mbf{z}_3)=\{\oplus\}$.
  However, it might be the case that the user would be content to
  accept a more general explanation for
  $\fml{T}=\{\oplus,\oplus\oplus\}$.
\end{example}

\jnoteF{Illustrate difference with DT example.}

Observe that considering missing inputs with several classes
in $\fml{T}$, allows for a more general (still rigorous) explanation,
able to explain all the classes in $\fml{T}$ (simultaneously).
Otherwise, for each of the classes in $\fml{T}$, the user would
have to obtain completely specified instances (for example by
imputation of the missing values), and only then compute
explanations (obtaining one explanation for each of the
specified instances, i.e, one explanation for each class in $\fml{T}$).
The user would end up with a set of explanations, instead of only one.

\subsection{Explaining $\fml{T}$}

Depending on the selected set of specified features, such set may or
may not suffice for a given prediction, i.e.\ there are not (or there
are) assignments to the non-specified features that cause the
prediction to change.
\begin{definition}
Given $\mbf{z}\in\mbb{F}'$, and associated partition of $\fml{F}$,
$\fml{U}$ and $\fml{S}$, we say that $\mbf{z}$ is \emph{sufficient for
the prediction} $\fml{T}$ if,
\begin{equation} \label{eq:suffpred}
  \forall(\mbf{x}\in\mbb{F}).\left(\land_{i\in\fml{S}}(x_i=z_i)\right)\limply\left(\kappa(\mbf{x})\in\fml{T}\right)
\end{equation}
If~\eqref{eq:suffpred} does not hold, then $\mbf{z}$ is not sufficient
for the prediction.
\end{definition}

Observe that $\fml{T}$ can either be computed from $\mbf{z}$, or it
can be user-specified. Also, if $\fml{T}$ is computed from $\mbf{z}$
it may deemed interesting or non-interesting by a human decision
maker.

In the remainder of the paper, we focus on the cases where $\mbf{z}$
is sufficient for the prediction given $\fml{T}$. Otherwise, the
model and $\mbf{z}$ do not suffice to answer ''Why?''~and
''Why~Not?''~questions regarding the prediction, i.e.\ we would have to
impute values to the features in $\fml{U}$.
Also, we could conceivably expand $\fml{T}$ so that $\mbf{z}$ would become sufficient
for the prediction.

%
Given the definition of $\mbf{z}$ that is sufficient for the
prediction, a set $\fml{X}$ of (specified) features,
$\fml{X}\subseteq\fml{S}$, is an abductive explanation (AXp) if
$\fml{X}$ is sufficient for the prediction (in this case $\fml{T}$)
and $\fml{X}$ is irreducible.
Thus, an AXp is a minimal set $\fml{X}\subseteq\fml{S}\subseteq{F}$
such that,
\begin{equation}
  \forall(\mbf{x}\in\mbb{F}).\left(\land_{i\in\fml{X}}(x_i=z_i)\right)\limply\left(\kappa(\mbf{x})\in\fml{T}\right)
\end{equation}
Any set $\fml{X}$ respecting the statement above is a \emph{weak} AXp
(or WAXp).
Similarly, we can define a contrastive explanation given some
(specified) features.
A CXp is a minimal set $\fml{Y}\subseteq\fml{S}\subseteq{F}$ such
that,
\begin{equation}
  \exists(\mbf{x}\in\mbb{F}).\left(\land_{i\in\fml{S}\setminus\fml{Y}}(x_i=z_i)\right)\land\left(\kappa(\mbf{x})\not\in\fml{T}\right)
\end{equation}
i.e.\ the features in $\fml{Y}$ (and also in $\fml{U}$) are allowed
to take any value in their domain, and this suffices for obtaining a
prediction not in $\fml{T}$.
Any set $\fml{Y}$ respecting the statement above is a \emph{weak} CXp
(or WCXp).
We use predicates $\waxp$ and $\wcxp$ to refer to sets that may or
may not satisfy the definitions of weak AXp/CXp.

One key observation is that the definitions above continue to exhibit
important properties of explanations that have been identified in
earlier work.

\begin{proposition}[Monotonicity] \label{prop:monomf}
  Given $\mbf{z}$ (and so a partition of $\fml{F}$, $\fml{S}$ and
  $\fml{U}$) that is sufficient for the prediction
  $\fml{T}\subsetneq\fml{K}$, then $\waxp$ and $\wcxp$ are monotone
  and up-closed.
\end{proposition}

\begin{proof}\footnote{Due to lack of space, most of the proofs are
  included in the supplementary materials.}
  By definition of $\mbf{z}$ being sufficient for the prediction, and
  because $\kappa$ is surjective, then it is immediate that
  $\waxp(\fml{S})$ and $\wcxp(\fml{S})$ hold. 
  For both $\waxp$ and $\wcxp$, it is also plain that if the
  predicates hold for some set $\fml{I}\subseteq\fml{S}$, then the
  predicates will hold for any
  $\fml{I}\subseteq\fml{I}'\subseteq\fml{S}$.
  In the case of weak AXp's, if $\fml{I}'\supseteq\fml{I}$, then any
  point $\mbf{x}\in\mbb{F}$ consistent with $\fml{I}'$ will also be
  consistent with $\fml{I}$.
   Since $\waxp(\fml{I})$ holds, then
  $\waxp(\fml{I}')$ also holds, i.e.\ the prediction does not change.
  In the case of weak CXp's, if $\fml{I}'\supseteq\fml{I}$, then any
  point $\mbf{x}\in\mbb{F}$ consistent with $\fml{I}$ will also be
  consistent with $\fml{I}$. Since $\wcxp(\fml{I})$ holds, then
  $\wcxp(\fml{I}')$ also holds, i.e.\ the prediction can be changed.
\end{proof}

This property allows defining AXp's/CXp's using~\eqref{eq:axp3}
and~\eqref{eq:cxp3} as before but considering the
generalizations proposed as follows:
\begin{align}
  \axp(\fml{X}) & :=
  \waxp(\fml{X})
  \land_{t\in\fml{X}}\neg\waxp(\fml{X}\setminus\{t\})
  \\
  \cxp(\fml{Y}) & :=
  \wcxp(\fml{Y})
  \land_{t\in\fml{Y}}\neg\wcxp(\fml{Y}\setminus\{t\})
\end{align}

Let $\mbf{z}\in\mbb{F}'$ (with associated partition of $\fml{F}$ into
$\fml{U}$ and $\fml{S}$), and let $\fml{S}$ be sufficient for 
$\fml{T}\subseteq\fml{K}$. Then, for the generalized explanation
problem $\fml{E}'=(\fml{C},(\mbf{z},\fml{T}))$ we define,
\begin{align}
  \mbb{A}_{\mfrak{u}}(\fml{E}') & = \{\fml{X}\subseteq\fml{S}\,|\,\axp(\fml{X})\}\\
  \mbb{C}_{\mfrak{u}}(\fml{E}') & = \{\fml{Y}\subseteq\fml{S}\,|\,\cxp(\fml{Y})\}
\end{align}

Given the above, the important property of duality between
explanations can be generalized as follows.
\begin{restatable}[Duality relation]{proposition}{propdualmf}
  \label{prop:dualmf}
  Each AXp in $\mbb{A}_{\mfrak{u}}(\fml{E}')$ is a minimal hitting set
  (MHS) of the CXp's in $\mbb{C}_{\mfrak{u}}(\fml{E}')$ and
  vice-versa.
\end{restatable}


\begin{example} \label{ex:runex04}
  With respect to the DT of~\cref{fig:runex},
  with ~$\mbf{z}=(35.0,0.25,\mfrak{u},100.0,\mfrak{u},5.0)$
  from~\cref{ex:runex02},
  we can identify by inspection weak AXp's and also plain AXp's.
  %
  %
  A possible explanation is of course $\fml{S}=\{1,2,4,6\}$, i.e.\ the
  non-missing feature values. However, the paper's results allow us to
  compute far simpler (and still rigorous) explanations. These include
  $\{4,6\}$ (representing the literals $\tn{PLT}=100$ and
  $\tn{WBC}=5.0$), and also $\{1,4\}$ (representing the literals
  $\tn{HCT}=35.0$ and $\tn{PLT}=100$).
  Given the set of AXp's
  $\mbb{A}_{\mfrak{u}}(\fml{E}')=\{\{4,6\},\{1,4\}\}$ then, 
  by MHS duality, we obtain
  $\mbb{C}_{\mfrak{u}}(\fml{E}')=\{\{4\},\{1,6\}\}$.
  As an example, if we allow $\tn{PLT}$ to take any value, then by
  picking $\tn{PLT}>193$ and by setting $\tn{Neutr.}$ (which is
  unspecified) to a suitable value (e.g.\ $>4.9$), we manage to change
  the prediction to a value not in $\{\oplus,\oplus\oplus\}$.
\end{example}


%

\subsection{Computing Explanations}

For some logic theory $\mbb{T}$, we use $\llbracket\cdot\rrbracket$ to
denote the encoding of a logic statement into statements of $\mbb{T}$.
The definition of AXp (or CXp) can be used to devise a generic
approach for computing one AXp (or CXp)~\cite{msi-aaai22}.
Given a target set of features $\fml{W}$ (instead of $\fml{X}$,
\eqref{eq:axp} (resp.~\eqref{eq:cxp}) can be reformulated as the
following statement being false (resp.~true): 
\begin{equation} \label{eq:consistent}
  \consistent\left(\left\llbracket\left(\bigwedge\nolimits_{i\in\fml{W}}(x_i=z_i)\right)\land(\kappa(\mbf{x})\not\in\fml{T})\right\rrbracket\right)
\end{equation}
where $\fml{T}=\kappa'(\mbf{z})$. 

Furthermore, one can define abstract predicates for iteratively
computing AXp's/CXp's, as follows:
\begin{align} \label{eq:predaxp}
  \predicate_{\tn{axp}}(\fml{W};&\mbb{T},\fml{F},\fml{S},\kappa,\mbf{z},\fml{T})
  \triangleq \nonumber \\
  \neg\consistent&\left(\left\llbracket\left(\bigwedge\limits_{i\in\fml{W}}(x_i=z_i)\right)\land(\kappa(\mbf{x})\not\in\fml{T})\right\rrbracket\right)
\end{align}
\begin{align} \label{eq:predcxp}
  \predicate_{\tn{cxp}}(\fml{W};&\mbb{T},\fml{F},\fml{S},\kappa,\mbf{z},\fml{T})
  \triangleq \nonumber\\
  \consistent&\left(\left\llbracket\left(\bigwedge\limits_{i\in\fml{S}\setminus\fml{W}}(x_i=z_i)\right)\land(\kappa(\mbf{x})\not\in\fml{T})\right\rrbracket\right)
\end{align}

\jnoteF{Include simple extraction algorithm, starting from seed.}

Given the above, \Cref{alg:onexp} illustrates the computation of one
AXp or one CXp. The predicate to be used, for computing either one
AXp or one CXp, represents part of the algorithm's parameterization.
\begin{algorithm}[t]
  \input{./algs/onexp}
  \caption{Finding one AXp/CXp} \label{alg:onexp}
\end{algorithm}

\begin{restatable}{proposition}{proponexp} \label{prop:onexp}
  For completely specified inputs, let the complexity of deciding
  whether $\fml{W}\subseteq\fml{F}$ is an AXp/CXp be $\mfrak{C}$.
  Then, for partially specified inputs sufficient for the prediction
  $\fml{T}$, the complexity of deciding whether $\fml{W}\in\fml{F}$ is
  an AXp/CXp is $\mfrak{C}$.
\end{restatable}

\begin{proof}(Sketch)
  Deciding whether $\fml{W}$ is an AXp/CXp involves checking
  (in)consistency of changing the prediction when some of the features
  are fixed. As shown in~\cref{alg:onexp}, it suffices to change the
  seed that bootstraps the AXp/CXp extraction algorithm. As can be
  concluded, this change to the AXp/CXp extraction algorithm is
  orthogonal to the predicate used for deciding the weak AXp/CXp
  condition, and so can be used with any of the classifiers studied in
  recent years. Observe that the same argument can be used with
  polynomial time explanation
  algorithms~\cite{msgcin-icml21,iims-jair22}. 
\end{proof}

The same observations yield the following result.

\begin{corollary}
  Consider partial specified inputs sufficient for a prediction
  $\fml{T}$.
  Then, the complexity of computing one AXp/CXp remains unchanged
  independently of whether some features are incompletely specified.
\end{corollary}

\begin{proof}
  As shown in~\cref{alg:onexp}, the main change to the algorithm is to
  pick seeds that do not involve the non-specified features.
\end{proof}

MARCO-like algorithms~\cite{lpmms-cj16} for the enumeration of both
AXp's and CXp's have been proposed in recent
work~\cite{inams-aiia20,msgcin-icml21}. 
\cref{alg:allxp} shows a modified MARCO-like algorithm for the
enumeration of AXp's/CXp's when some inputs are unspecified.
(Unspecified features are allowed to take any value from their domain,
and so are treated as universal features.)
\begin{algorithm}[t]
  \input{./algs/allxp}
  \caption{Finding all AXp's/CXp's} \label{alg:allxp}
\end{algorithm}
The enumeration of AXp's/CXp's builds on the general purpose algorithm
for computing one AXp/CXp from a seed (see~\cref{alg:onexp})%
\footnote{%
The supplementary material include examples of
executing~\cref{alg:onexp} and~\cref{alg:allxp} on the example DT
of~\cref{fig:runex}.}.
The SAT oracle returns an outcome $\top$ (or true) or $\bot$ (or
false), and an assignment $\mbf{u}$ in case the outcome is $\top$.

\jnoteF{Include simple enumeration algorithm.}

\begin{corollary}
  Consider partial specified inputs sufficient for a prediction
  $\fml{T}$.
  Then, the complexity of enumerating AXp's/CXp's remains unchanged
  independently of whether some features are incompletely specified. 
\end{corollary}

\begin{proof}(Sketch)
  As shown in~\cref{alg:allxp}, the main change to the algorithm is to
  pick seeds that do not involve 
  non-specified features.
\end{proof}

\jnoteF{Are there special cases, e.g.\ NBCs, d-DNNFs, XpG's?\\
  \tbf{Note:} I've check d-DNNFs and XpG's and I'm certain that the
  proposed algorithms should work with a starting (incompletely
  specified) seed. I'm also fairly convinced about the NBCs algorithm.
  And of course we know it also works for monotonic.\\
  The proof has been edited to clarify why the algorithms work in all
  cases.}

\jnoteF{%
  Claim:\\
  For poly-time algorithms for computing one AXp which refine an
  over approximation of an AXp, $\fml{S}$ can just be viewed as the
  seed that bootstraps the algorithm.
}

%% file: algs/onexp.tex
\begin{flushleft}
  \hspace*{\algorithmicindent}
  \textbf{Input}: {
    Seed $\fml{R}\subseteq\fml{S}$,
    parameters $\predicate$,
    $\mbb{T}$, $\fml{F}$, $\fml{S}$, $\kappa$, $\mbf{z}$, $\fml{T}$}\\
  \hspace*{\algorithmicindent}
  \textbf{Output}: {One XP $\fml{W}$}
\end{flushleft}

\begin{algorithmic}[1]
  \Procedure{$\onexp$}{$\fml{R};\mbb{P},\mbb{T},\fml{F},\fml{S},\kappa,\mbf{z},\fml{T}$}
  \State{$\fml{W} \gets \fml{R}$}
  \Comment{Initialization: $\mbb{P}(\fml{W})$ holds}
  \For{$i\in\fml{R}$}
  \Comment{Loop invariant: $\mbb{P}(\fml{W})$ holds}
    \If{$\predicate(\fml{W}\setminus\{i\}; \mbb{T},\fml{F},\fml{S},\kappa,\mbf{z},\fml{T})\,{=}\,\top$}
    \State{$\fml{W} \gets \fml{W}\setminus\{i\}$}
    \Comment{If $\mbb{P}(\fml{W}\setminus\{i\})$, update $\fml{W}$}
    \EndIf
  \EndFor
  \State{\bfseries{return}~{$\fml{W}$}}
  \Comment{Returned set $\fml{W}$: $\mbb{P}(\fml{W})$ holds}
\EndProcedure
\end{algorithmic}
%
%

%% file: algs/allxp.tex
\begin{flushleft}
  \hspace*{\algorithmicindent}
  \textbf{Input}: {
    Parameters $\predaxp$, $\predcxp$,
    $\mbb{T}$, $\fml{F}$, $\fml{S}$, $\kappa$, $\mbf{v}$, $\fml{T}$} 
  \hspace*{\algorithmicindent}
  %
\end{flushleft}

\begin{algorithmic}[1]
  \State{%
    \label{alg:exp:ln01}$\fml{H}\gets\emptyset$}%
  \Comment{$\fml{H}$ defined on set $U=\{u_1,\ldots,u_m\}$}
  \Repeat\label{alg:exp:ln02} 
  \State{%
    \label{alg:exp:ln03}$(\outc,\mbf{u})\gets\SAT(\fml{H})$}
  \If{\label{alg:exp:ln04}$\outc\,{=}\,\top$}
  \State{\label{alg:exp:ln05}$\fml{R}\gets\{i\in\fml{S}\,|\,u_i=0\}$}%
  \Comment{$\fml{R}$: seed, \emph{fixed} features}
  \State{\label{alg:exp:ln06}$\fml{Q}\gets\{i\in\fml{S}\,|\,u_i=1\}$}%
  \Comment{$\fml{Q}$: \emph{universal} features}
  \If{\label{alg:exp:ln07}$\predcxp(\fml{Q};\mbb{T},\fml{F},\fml{S},\kappa,\mbf{v},\fml{T})\,{=}\,\top$}
  \Comment{$\fml{Q}\;\!{\supseteq}\;\!\tn{CXp}$} 
  \State{\label{alg:exp:ln08}$\fml{P}\gets\onexp(\fml{Q};\predcxp,\mbb{T},\fml{F},\fml{S},\kappa,\mbf{v},\fml{T})$}
  \State{\label{alg:exp:ln09}$\prtcxp(\fml{P})$}
  \State{\label{alg:exp:ln10}$\fml{H}\gets\fml{H}\cup\{(\lor_{i\in\fml{P}}\neg{u_i})\}$}
  \Else\label{alg:exp:ln11}
  \Comment{$\fml{R}\,{\supseteq}\,\tn{AXp}$}
  \State{\label{alg:exp:ln12}$\fml{P}\gets\onexp(\fml{R};\predaxp,\mbb{T},\fml{F},\fml{S},\kappa,\mbf{v},\fml{T})$}
  \State{\label{alg:exp:ln13}$\prtaxp(\fml{P})$}
  \State{\label{alg:exp:ln14}$\fml{H}\gets\fml{H}\cup\{(\lor_{i\in\fml{P}}{u_i})\}$}
  %
  \EndIf
  \EndIf
  \Until{\label{alg:exp:ln15}$\outc\,{=}\,\bot$}
\end{algorithmic}

%% file: xtra.tex
\section{Additional Results} \label{sec:xtra}

This section uncovers additional properties of explanations, namely
when inputs are partially specified. These properties allow relating
explanations with different subsets of unspecified features, and can
be used for example for the on-demand enumeration of explanations.
Furthermore, we also briefly investigate extensions of the basic
explainability framework.

\subsection{Properties of Explanations} \label{ssec:xpprop}

Throughout this section we consider a generalized explanation problem
$\fml{E}'_r=(\fml{C},(\mbf{z},\fml{T}))$, with
$\kappa'(\mbf{z})=\fml{T}$. The index $r$ denotes an order relation
on the number of unspecified features, i.e.\ for $\fml{E}'_{r_1}$ and
$\fml{E}'_{r_2}$, with $r_1<r_2$, then $\fml{E}'_{r_1}$ has no more
unspecified features than $\fml{E}'_{r_2}$.
We start by investigating properties of AXp's and CXp's for
explanation problems that refine $\fml{E}'_r$, i.e.\ the set of
specified features $\fml{S}_j$ is contained in $\fml{S}$. Afterwards,
we investigate properties of AXp's and CXp's for explanation problems
that relax $\fml{E}'$, i.e.\ the set of specified features $\fml{S}_j$
contains $\fml{S}_r$.

%
Let $(\mbf{v},c)$ be such that $\mbf{v}\sqsubseteq\mbf{z}$,
$c\in\fml{T}$, and $\kappa(\mbf{v})=c$ ($\kappa'(\mbf{v})=\{c\}\subseteq \fml{T}$).
Moreover, let $\fml{E}'_0=(\fml{C},(\mbf{v},\fml{T}))$ be an
explanation problem given $\fml{E}'_r$ and concretized $\mbf{v}$.
%
We can thus establish the following:
\begin{restatable}{proposition}{proputonou} \label{prop:utonou}
  Given the definitions of $\fml{E}'_r$ and $\fml{E}'_0$ above, then:
  \begin{enumerate}[nosep]
  \item If $\fml{X}\in\mbb{A}_{\mfrak{u}}(\fml{E}'_r)$, then
    $\fml{X}\in\mbb{A}_{\mfrak{u}}(\fml{E}'_0)$,
    i.e.\ $\fml{X}$ is an AXp for $\fml{E}'_0$ and
    $\mbb{A}_{\mfrak{u}}(\fml{E}'_r)\subseteq\mbb{A}_{\mfrak{u}}(\fml{E}'_0)$.
  \item If $\fml{Y}\in\mbb{C}_{\mfrak{u}}(\fml{E}'_r)$, then there
    exists $\fml{Z}\subseteq\fml{U}$ such that
    $\fml{W}=\fml{Y}\cup\fml{Z}\in\mbb{C}_{\mfrak{u}}(\fml{E}'_0)$,
    i.e.\ each $\fml{Y}\cup\fml{Z}$ is a CXp for $\fml{E}'_0$ and
    $\forall(\fml{Y}\in\mbb{C}_{\mfrak{u}}(\fml{E}'_r))\exists(\fml{W}\in\mbb{C}_{\mfrak{u}}(\fml{E}'_0)).\left[\fml{Y}\subseteq\fml{W}\right]$.
  \end{enumerate}
\end{restatable}

\paragraph{Nested duality.}
\cref{prop:dualmf} and \cref{prop:utonou} reveal an important property
of duality among AXp's and CXp's that holds in general.
The rest of this section assumes the following explanation scenario.
We consider the computation of an AXp $\fml{X}$ of some explanation
problem $\fml{E}'$, starting from $\mbf{z}$, that partitions $\fml{F}$
into $\fml{S}=\fml{S}_r$ and $\fml{U}=\fml{U}_r$.
Moreover, let $r\le{j}\le{r+s}$.
%
In the process of computing the AXp $\fml{X}$, let the computed weak
AXp's be given by the sequence
$\langle\fml{X}_r=\fml{S}_r,\fml{X}_{r+1},\ldots,\fml{X}_{r+s}=\fml{X}\rangle$.
(At each step of executing the AXp extraction algorithm, the set of
fixed features is $\fml{X}_j$.) 
Since for each $\fml{X}_j$, 
some features are left unspecified, then define the explanation
problem $\fml{E}'_j$, 
with $\fml{E}'_r=\fml{E}'$, exhibiting a corresponding set of AXp's
and CXp's,
$\mbb{A}_{\mfrak{u}}(\fml{E}'_j)$ and $\mbb{C}_{\mfrak{u}}(\fml{E}'_j)$.
For each $\fml{E}'_j$, the specified features are
$\fml{S}_j=\fml{X}_j$ and the unspecified features are
$\fml{U}_j=\fml{F}\setminus\fml{X}_j$. 
(Clearly, for $j=r$, $\fml{S}_r=\fml{S}$ and $\fml{U}_r=\fml{U}$.)
%
Given the explanation problem $\fml{E}'$, and the sequence
$\langle\fml{X}_r=\fml{S},\fml{X}_{r+1},\ldots,\fml{X}_{r+s}=\fml{X}\rangle$,
then the following results hold.

\begin{restatable}[Relating AXp's and
    CXp's]{proposition}{propaxpsvscxps} \label{prop:axps_vs_cxps}
  Given the assumptions above, it is the case that,
  \begin{enumerate}
  \item If $\fml{W}\in\mbb{A}_{\mfrak{u}}(\fml{E}'_{k})$, then
    $\fml{W}\in\mbb{A}_{\mfrak{u}}(\fml{E}'_j)$, for any $j,k$ with
    $r\le{j}\le{k}\le{r+s}$;
  \item If $\fml{Y}\in\mbb{C}_{\mfrak{u}}(\fml{E}'_k)$, then there
    exists $\fml{Z}_r\subseteq\fml{X}_r$ such that
    $\fml{Y}\cup\fml{Z}_r\in\mbb{C}_{\mfrak{u}}(\fml{E}'_{j})$, for
    any $j,k$ with $r\le{j}\le{k}\le{r+s}$.
  \end{enumerate}
\end{restatable}

Moreover, we also have the following duality result.

\begin{restatable}[Nested Duality]{proposition}{propndual}
  \label{prop:ndual}
  Each AXp in $\mbb{A}_{\mfrak{u}}(\fml{E}'_j)$ is a minimal
  hitting set of the CXp's in $\mbb{C}_{\mfrak{u}}(\fml{E}'_j)$ and
  vice-versa, for $r\le{j}\le{r+s}$. 
\end{restatable}

\input{./tabs/tab-res}

\subsection{Generalizations}

\paragraph{Necessity, Relevancy and Irrelevancy.}

A generalization that can be made is with regards to querying
the participation of features in explanations.
Namely the queries of necessity and relevancy~\cite{hcmpms-tacas23},
that is, whether a feature participates in all the explanations of the
output class (the feature is necessary), or the feature belongs to at
least one explanation of the output classification class (the feature is
relevant).

\begin{algorithm}[t]
  \input{./algs/gen-frp}
  \caption{Deciding feature relevancy} \label{alg:gen-frp}
\end{algorithm}

Similar to the case of computing an explanation, querying for necessity
or relevancy of features for explanations does not alter the complexity
of the problem by the inclusion of unspecified inputs.
Consider a classifier for which we wish to query
the relevancy of a feature considering explanations with unspecified inputs.
\cref{alg:gen-frp} proposes an adaptation of the algorithm
in~\cite{hcmpms-tacas23} taking into account unspecified inputs.

\jnoteF{Constraints on the inputs}

\paragraph{Input constraints \& probabilistic explanations.}
Logic-based explanations often assume a uniform distribution over the
inputs, i.e.\ all points in feature space are equally likely and
possible. Recent work proposes to account for explicit constraints on
the inputs~\cite{rubin-aaai22} 
by considering in the definition of a weak AXp an additional map
$\varsigma:\mbb{F}\rightarrow\{0,1\}$ constraining the point
(0 disallows the point, and 1 allows it).
With this modified definition of $\waxp$, it is immediate that, given a generalized explanation problem $\fml{E}'=(\fml{C},(\mbf{z},\fml{T}))$, a $\waxp$ with input constraints can be generalized as follows:
\begin{align} \label{eq:axp4}
  \waxp(\fml{X}) &\! :=\! \forall\mbf{x}\in\mbb{F}.
  \left[
    \varsigma(\mbf{x})\land\bigland\limits_{i\in{\fml{X}}}(x_i=z_i)
    \right]
  \limply \kappa(\mbf{x})\in \fml{T}
\end{align}

\jnoteF{Generalized literals, from~\cite{iims-jair22}?}

\jnoteF{Probabilistic explanations?\\
  Locally minimal PAXp's should in principle be computable from seed.}

%
In a similar vein, it is apparent (and so it is conjectured) that the
computation of probabilistic
explanations~\cite{kutyniok-jair21,ihincms-arxiv23,barcelo-nips22} can be handled in a
similar way.
Recently, \cite{ihincms-arxiv23} proposed the computation of Probabilistic Abductive Explanations.
The authors introduce the concept of weak Probabilistic Abductive Explanation $\fml{X}$ ($\wpaxp(\fml{X})$), which represents the probability of a variable $x$ having the same classification as a given instance $\mbf{v}$, conditioned to $x$ agreeing with the values of $\mbf{v}$ in the features of $\fml{X}$, being larger than a given $\delta$.
Given a generalized explanation problem $\fml{E}'=(\fml{C},(\mbf{z},\fml{T}))$,
the definition of $\wpaxp(\fml{X})$ can be generalized as follows:
\begin{align}\label{eq:wpaxpui}
  \wpaxpui(\fml{X})& :=~  Pr_x(\kappa(x)\in\fml{T}  ~|~x_\fml{X} = z_\fml{X}) \ge \delta
\end{align}

\paragraph{Further generalizations.}
The generalization of predicted class proposed in~\cref{sec:xpsi}
requires that the classifier computes a total function, i.e.\ each
point in feature space maps to a concrete class.
However, the results from~\cref{sec:xpsi} but also
from~\cref{ssec:xpprop} can be generalized to a more general set up
where $\kappa$ can associate multiple classes for each point in
feature space, but also when the class is left unspecified,
i.e.\ $\kappa$ is a map from $\mbb{F}$ into $2^{\fml{K}'}$, with
$\fml{K}'=\fml{K}\cup\mfrak{u}$, and where $\mfrak{u}$ denotes any
class.
For example, this generalization allows for assigning \emph{don't
care} classes to points in feature space which are known not to be
feasible, offering an alternative to recent work~\cite{rubin-aaai22}.

%% file: tabs/tab-res.tex
\begin{table*}[t]
  \begin{center}
    \begin{tabular}{ccccccc} 
      \toprule
      Dataset & $|\tn{DT}|$ & Avg Depth & Term.~Node & Target / Unwanted Classes &
      Specified Features & AXp 
      \\ \toprule
      %
      %
      \texttt{dermatology} & 51 & 6.2 & 32 & $\{1,2,4,5,6\}/\{3\}$ &
      $\{\tn{f4},\tn{f7},\tn{f13},\tn{f19},\tn{f26},\tn{f27},\tn{f33}\}$ &
      $\{\tn{f7},\tn{f26}\}$ 
      \\ \midrule
      %
      %
      \texttt{student-por} & 287 & 7.5 & 48 & $\{14..19\}/\{0..13\}$ &
      $\{\tn{f2},\tn{f9},\tn{f12},\tn{f24},\tn{f29},\tn{f31}\}$ &
      $\{\tn{f31}\}$ 
      \\ \midrule
      %
      %
      \texttt{auto} & 33 & 4.4 & 26 & $\{0..3\}/\{\tn{-3..-1}\}$ &
      $\{\tn{f5},\tn{f8},\tn{f12},\tn{f18}\}$ &
      $\{\tn{f8},\tn{f18}\}$ 
      \\ \midrule
      %
      %
      \texttt{soybean} & 109 & 6.8 & 106 & $\{0,5..18\}/\{1,2,3,4\}$ &
      $\{\tn{f1},\tn{f2},\tn{f14},\tn{f25},\tn{f26},\tn{f28},\tn{f29}\}$ &
      $\{\tn{f14},\tn{f28}\}$ 
      \\ \midrule
      %
      %
      \texttt{zoo} & 41 & 5.3 & 26 & $\{2..7\}/\{1\}$ &
      $\{\tn{f0},\tn{f4},\tn{f6},\tn{f8},\tn{f10},\tn{f12},\tn{f15}\}$ &
      $\{\tn{f0},\tn{f15}\}$ 
      \\ 
      \bottomrule
    \end{tabular}
    \caption{Summary of results} \label{tab:res}
  \end{center}
\end{table*}

%% file: algs/gen-frp.tex
\begin{flushleft}
  \hspace*{\algorithmicindent}
  \textbf{Input}: {
    Target feature $t\in\fml{S}$,
    parameters $\fml{F}$, $\fml{S}$, $\kappa$, $\mbf{z}$, $\fml{T}$}\\
  %
\end{flushleft}
\begin{algorithmic}[1]
  \Procedure{$\frp$}{$t;\fml{F},\fml{S},\kappa,\mbf{z},\fml{T}$}
    \State{$\fml{H} \gets \{(s_t)\}$} \Comment{$\vars(\fml{H})=\{s_i\,|\,i\in\fml{S}\}$}
    \Repeat
      \State{$(\outc,\mbf{s})\gets\SAT(\fml{H})$}
      \If{$\outc=\True$}
        \State{$\fml{P} \gets \{i \in \fml{S} \,|\, s_i=1 \}$}
        \State{$\fml{D} \gets \{i \in \fml{S} \,|\, s_i=0 \}$}
        \If{$\neg \waxpui(\fml{P})$}
          \State{$\fml{H} \gets \fml{H}\cup\{(\lor_{i\in{\fml{D}}}s_i)\}$}
          \Comment{Block non-$\waxpui$}
        \ElsIf{$\neg \waxpui(\fml{P} \setminus \{t\})$}
            \State{\Return \True}
            \Comment{$t$ is relevant}
        \Else \Comment{ Block $t$-irrelevant $\waxpui$}
          \State{$\fml{H} \gets \fml{H} \cup \{(\lor_{i\in{\fml{P}\setminus\{t\}}} \neg s_i)\}
          $}
          
        \EndIf
      \EndIf
    \Until{$\outc=\False$}
    \State{\Return \False}
    \Comment{$t$ is irrelevant}
  \EndProcedure
\end{algorithmic}

%% file: res.tex
\section{Experiments} \label{sec:res}

\jnoteF{Use sets of classes to obtain smaller and so more intuitive
  explanations.}

\jnoteF{The use of unspecified inputs should simplify the enumeration
  of explanations. That is another advantage of our work.}

\paragraph{Setup.}
To illustrate use cases of the logic-based explainability framework
proposed in this paper, we consider an explainer for decision
trees~\cite{iims-jair22}. In contrast with earlier
work~\cite{iims-jair22}, our explainer navigates the DT and does not
encode to propositional Horn clauses. As in earlier work, the
running times for computing one AXp/CXp are always negligible, since
the explainer runs in polynomial time on the tree size.
Also, the explainer has been implemented in Python. 
The explainer allows a user to declare unspecified features as well
as to list accepted \emph{target} classes (all other classes are referred to
as \emph{unwanted}).

\paragraph{Target/Unwanted classes.}
The goal of our first experiment is to compute smaller (and so easier
to understand~\cite{miller-pr56}) AXp's, such that a given set of
target classes is guaranteed to occur. For example, we
may want to explain why there is some degree of risk associated with
an insurance, but we do not necessarily care about the actual degree
of risk.

Furthermore, the experiments consider example decision trees, obtained
with a commercial tool\footnote{%
We opted to use the optimal decision trees induced with the tree
learner from Interpretable AI~\cite{bertsimas-ml17}, which is
available (on request) from \url{https://www.interpretable.ai/}.}.
A maximum depth of 8 was chosen, so as to keep the trees sufficiently
shallow.
Furthermore, we consider a number of well-known publicly available
datasets,
namely \texttt{dermatology},
\texttt{student-por},
\texttt{auto},
\texttt{soybean},
and
\texttt{zoo}~\footnote{The datasets were obtained from the UCI
repository (\url{https://archive.ics.uci.edu/ml/}), in the case of
\texttt{dermatology}, \texttt{student-por}, \texttt{soybean} and
\texttt{zoo}, and from the PennML repository
(\url{https://epistasislab.github.io/pmlb/}), in the case of
\texttt{auto}. Motivated by the experiments in this section, we opted
for multi-class datasets.},
for each of which all features are categorical\footnote{%
For the purpose of the paper, we will not dwell on the meaning
ascribed to each feature, i.e.\ features will be denoted by $\tn{f}i$,
with $i$ ranging over the number of features.}
Moreover, for each dataset, two sets of (resp.~target/unwanted)
classes were chosen.
  Concretely, for \texttt{dermatology} the unwanted
  class was 3 (i.e.\ \emph{lichen planus}). For \texttt{student-por}
  the unwanted classes were all the grades between 0 and 13 (i.e.\ the
  target are good or above grades). For \texttt{auto} the unwanted
  classes were those of no risk (i.e.\ -3, -2 and -1, but it should be
  noted that the dataset does not have entries predicting -3 or -2).
  For \texttt{soybean}, the unwanted classes were rot diseases
  (i.e. 1, 2, 3 and 4).
  Finally, for \texttt{zoo}, the unwanted class was \emph{mammal},
  i.e.\ we are interested in understanding which features suffice for
  not predicting a \emph{mammal}.
%
In all cases, and given a target terminal node in the DT, the
specified features are selected as those tested in the path from the
root to the terminal node. The unspecified features are all remaining
ones.

%
\cref{tab:res} summarizes the results of computing one AXp.
As can be observed, starting from a non-negligible number of features
(i.e.\ a weak AXp given the unspecified features), we are able to
obtain fairly small AXp's, which explain what must not be changed so
that the unwanted classes are not predicted. Such small explanations
are important, since it is generally accepted that smaller
explanations are simpler for human decision makers to
comprehend~\cite{miller-pr56}.

\paragraph{Assessing an ML classifier.}
As a second experiment, we
detail a case study of a recently proposed
decision tree for non-invasive diagnosis of
Meningococcal Disease (MD)
meningitis~\cite[Figure~9]{belmonte-ieee-access20}.
In this DT, some features refer to
traits of a patient (e.g.\ age, gender, and place where she lives), and
some other refer to symptoms that the patient exhibits (e.g.\ being in
a coma, headache, seizures, vomiting, stiff neck, and existence of
petechiae).
For this DT, that we can see also in~\cref{fig:cstudy} of our technical appendix, we have used our explainer as well as selected
unspecified features to prove that patients exhibiting no symptoms can
be diagnosed with MD meningitis.
This example illustrates a novel use case of logic-based
explainability in assessing how realistic, and possibly misleading,
are the predictions of an ML classifier.

%

%
The proposed procedure operates as follows. For each leaf predicting
the class of interest (i.e.\ a diagnosis of MD meningitis) we identify
the features representing symptoms and, among these, we pick the ones
that are active (i.e.\ the patient exhibits the symptom).
We then declare the other features (non-active or non-symptom) as
specified, and leave the active symptom features as unspecified.
%
If the specified features are sufficient for the prediction, then we
know that
a patient can be diagnosed with MD meningitis despite not exhibiting
any symptoms. Notwithstanding, we proceed to compute an AXp, since
this offers an irreducible set of features sufficient for the unwanted
prediction.

For example, given the following path of the DT that predicts MD
meningitis for the patient:
$$
\begin{array}{rl}
    \langle &
             \textsf{(Age > 5)}=1,
            \textsf{Petechiae}=0,
            \textsf{Stiff Neck}=0,  \\
            &
            \textsf{Vomiting}=1,
            \textsf{Zone}=2,
            \textsf{Seizures}=0, \\
            &
            \textsf{Headache}=0,
            \textsf{Coma}=0,
            \textsf{Gender}=1 \quad \rangle
  \end{array}
$$
We leave as unspecified features only the active symptom (Vomiting)
to look for an AXP.
The  AXP we obtain is:
$$
  \{ \textsf{(Age > 5)}, \textsf{Gender}, \textsf{Headache}, \textsf{Zone}\}
$$
That is, the unique active symptom in this path (Vomiting) is not
really necessary to predict a patient as having MD meningitis,
as it is only necessary for the patient to have age > 5,
reside in zone 2, to exhibit no headache, and have gender 1 (male).
This behavior is obtained also with other paths of the DT, and they are
not easily spotted by a visual inspection of the DT.








%% file: relw.tex
\section{Related Work} \label{sec:relw}

Missing data occurs in a number of settings, either in existing
datasets, and so for the training of ML
models~\cite{shamir-icml08,graham-arp09,shamir-ml10,graham-bk12,peng-sp13,little-jpp14,little-ps18,rubin-bk19},
but also for deciding and explaining
predictions~\cite{bro-cils05,rubin-bk19,demoor-cmpb22,zeng-wiseml22}.
Earlier work on explainability focuses on value attribution to missing
feature values as an expected prediction~\cite{vandenbroeck-ijcai19,vandenbroeck-corr20}, 
or as a desired prediction~\cite{ike-corr23}.

In contrast, this paper studies abductive and contrastive explanations
in the presence of missing feature values. Furthermore, and to the
best of our knowledge, past work on logic-based
explanations~\cite{darwiche-ijcai18,inms-aaai19,darwiche-ecai20,marquis-kr20,marquis-kr21,msi-aaai22}
considers fully specified inputs.

\jnoteF{It would be great if we could include additional references and
  detail.\\
  Any volunteers?}

%% file: conc.tex
\section{Conclusions} \label{sec:conc}

This paper investigates logic-based explanations in the case some
features are left unspecified, i.e.\ there exists missing data.
This involves extending feature space and generalizing the definition
of classification function. In turn, the proposed solution enables
adapting existing explanation algorithms, both for computing abductive
and contrastive explanations.

The paper also proves a number of additional results related with the
relationship between abductive and contrastive explanations.
The experimental results highlight a number of practical uses, both of
the proposed solution for reasoning with partially specified inputs, but 
also for computing smaller (and so more intuitive) explanations.
More importantly, the experiments also highlight how logic-based
explanations can serve to reveal important issues with induced ML
models.

%% file: acks.tex
%
%
\ack
This work was supported by the AI Interdisciplinary Institute ANITI,
funded by the French program ``Investing for the Future -- PIA3''
under Grant agreement no.\ ANR-19-PI3A-0004, and by the H2020-ICT38
project COALA ``Cognitive Assisted agile manufacturing for a Labor
force supported by trustworthy Artificial intelligence'',
and funded by the Spanish Ministry of Science and Innovation (MICIN) 
under project PID2019-111544GB-C22/AEI/10.13039/501100011033, 
consolideted research group grant 2021-SGR-01615,
by a Maria Zambrano fellowship, 
and by a Requalification fellowship financed by Ministerio de 
Universidades of Spain and by European Union -- NextGenerationEU.
The problem was first suggested to the authors by R.\ Passos during
discussions about the EU project COALA.
JMS also acknowledges
the incentive provided by the ERC who, by not funding this research
nor a handful of other grant applications between 2012 and 2022, has
had a lasting impact in framing the research presented in this paper.

%% file: replbib.tex
\newtoggle{mkbbl}

%% file: togbbl.tex
\settoggle{mkbbl}{false}

%% file: appendix.tex
\newcommand{\appendixhead}%
{\centering\textbf{\huge Appendix -- Supplementary Materials}
\vspace{0.25in}}

\twocolumn[\appendixhead]
\appendix

\section{Additional Proofs}

\begin{lemma} \label{lemma:hs_vs_Xps_1}
If $\fml{Y}$ is a hitting set of $\mbb{A}_{\mfrak{u}}(\fml{E}')$, then $\fml{Y}$ is a weak CXp of $\fml{E}'$.
\end{lemma}

\begin{proof}
Suppose that $\fml{Y}$ is a hitting set of $\mbb{A}_{\mfrak{u}}(\fml{E}')$.
By contradiction suppose that:
$$
  \forall(\mbf{x}\in\mbb{F}).
  \left[
  \bigland\nolimits_{i\in {\fml{S} \setminus \fml{Y}}}(x_i = z_i)
  \right]
  \limply
  (\kappa(\mbf{x}) \in \fml{T})
$$
but, then this means that $\fml{S} \setminus \fml{Y}$ is a weak AXp of $\fml{E}'$, and there exists an AXp in $\fml{S} \setminus \fml{Y}$ that belongs to $\mbb{A}_{\mfrak{u}}(\fml{E}')$.
Since $\fml{Y}$ is a hitting set of $\mbb{A}_{\mfrak{u}}(\fml{E}')$ then $\fml{Y} \cap (\fml{S} \setminus \fml{Y}) \neq \emptyset$, which is impossible.
Thus:
$$
  \exists(\mbf{x}\in\mbb{F}).
  \left[
  \bigland\nolimits_{i\in {\fml{S} \setminus \fml{Y}}}(x_i = z_i)
  \right]
  \land
  (\kappa(\mbf{x}) \not\in \fml{T})
$$
that is, $\fml{Y}$ is a weak CXp of $\fml{E}'$.
\end{proof}

\begin{lemma} \label{lemma:hs_vs_Xps_2}
If $\fml{X}$ is hitting set of $\mbb{C}_{\mfrak{u}}(\fml{E}')$, then $\fml{X}$ is a weak AXp of $\fml{E}'$.
\end{lemma}

\begin{proof}
Suppose that $\fml{X}$ is hitting set of $\mbb{C}_{\mfrak{u}}(\fml{E}')$.
By contradiction suppose that:
$$
  \exists(\mbf{x}\in\mbb{F}).
  \left[
  \bigland\nolimits_{i\in {\fml{X}}} (x_i = z_i)
  \right]
  \land
  (\kappa(\mbf{x}) \not\in \fml{T})
$$
since $\fml{X} = \fml{S} \setminus (\fml{S} \setminus \fml{X})$, then it is the same as:
$$
  \exists(\mbf{x}\in\mbb{F}).
  \left[
  \bigland\nolimits_{i\in \fml{S} \setminus(\fml{S} \setminus \fml{X})} (x_i = z_i)
  \right]
  \land
  (\kappa(\mbf{x}) \not\in \fml{T})
$$
this means that $\fml{S} \setminus \fml{X}$ is a weak CXp of $\fml{E}'$, and there exists a CXp inside $\fml{S} \setminus \fml{X}$ that belongs to $\mbb{C}_{\mfrak{u}}(\fml{E}')$.
Since $\fml{X}$ is a hitting set of $\mbb{C}_{\mfrak{u}}(\fml{E}')$ then $\fml{X} \cap (\fml{S} \setminus \fml{X}) \neq \emptyset$, which is impossible.
Thus:
$$
  \forall(\mbf{x}\in\mbb{F}).
  \left[
  \bigland\nolimits_{i\in \fml{X}}( x_i = z_i)
  \right]
  \limply
  (\kappa(\mbf{x}) \in \fml{T})
$$
that is, $\fml{X}$ is a weak AXp of $\fml{E}'$.
\end{proof}


\propdualmf*

\begin{proof}
Consider $\fml{X} \in \mbb{A}_{\mfrak{u}}(\fml{E}')$, and $\fml{Y} \in \mbb{C}_{\mfrak{u}}(\fml{E}')$.
Suppose by contradiction that $\fml{X} \cap \fml{Y} = \emptyset$.
Because $\fml{X} \in \mbb{A}_{\mfrak{u}}(\fml{E}')$, then it happens that:
$$
  \forall(\mbf{x}\in\mbb{F}).
  \left[
  \bigland\nolimits_{i\in{\fml{X}}}(x_i = z_i)
  \right]
  \limply
  (\kappa(\mbf{x}) \in \fml{T})
$$
Since $\fml{X} \cap \fml{Y} = \emptyset$ ( and both $\fml{X}, \fml{Y}\subseteq\fml{S}$), then $\fml{X} \subseteq \fml{S} \setminus \fml{Y}$, and:
$$
  \forall(\mbf{x}\in\mbb{F}).
  \left[
  \bigland\nolimits_{i\in {\fml{S} \setminus \fml{Y}}}(x_i = z_i)
  \right]
  \limply
  (\kappa(\mbf{x}) \in \fml{T})
$$
But this contradicts the fact that $\fml{Y} \in \mbb{C}_{\mfrak{u}}(\fml{E}')$.

Thus any $\fml{X} \in \mbb{A}_{\mfrak{u}}(\fml{E}')$ is a hitting set of any $\fml{Y} \in \mbb{C}_{\mfrak{u}}(\fml{E}')$, and vice-versa.

What is missing to prove is the minimality of the hitting sets.

Suppose that $\fml{Y} \in \mbb{C}_{\mfrak{u}}(\fml{E}')$ is a hitting set of $\mbb{A}_{\mfrak{u}}(\fml{E}')$.
Suppose by contradiction $\fml{Y}$ is not subset minimal hitting set, that is, there exists $\fml{Z} \subsetneq \fml{Y}$ such that $\fml{Z}$ is a hitting set of $\mbb{A}_{\mfrak{u}}(\fml{E}')$.
By Lemma~\ref{lemma:hs_vs_Xps_1}, then $\fml{Z}$ is a weak CXp of $\fml{E}'$, and there exists an CXp inside $\fml{Z}$, but this contradicts the subset minimality of the CXp $\fml{Y}$ because $\fml{Z} \subsetneq \fml{Y}$, thus $\fml{Y} \in \mbb{C}_{\mfrak{u}}(\fml{E}')$ is a minimal hitting set of $\mbb{A}_{\mfrak{u}}(\fml{E}')$.

Suppose now that $\fml{X} \in \mbb{A}_{\mfrak{u}}(\fml{E}')$ is a hitting set of $\mbb{C}_{\mfrak{u}}(\fml{E}')$.
Suppose by contradiction $\fml{X}$ is not subset minimal hitting set, that is, there exists $\fml{Z} \subsetneq \fml{X}$ such that $\fml{Z}$ is a hitting set of $\mbb{C}_{\mfrak{u}}(\fml{E}')$.
By Lemma~\ref{lemma:hs_vs_Xps_2}, then $\fml{Z}$ is a weak AXp of $\fml{E}'$, and there exists an AXp inside $\fml{Z}$, but this contradicts the subset minimality of the AXp $\fml{X}$ because $\fml{Z} \subsetneq \fml{X}$, thus $\fml{X} \in \mbb{A}_{\mfrak{u}}(\fml{E}')$ is a minimal hitting set of $\mbb{C}_{\mfrak{u}}(\fml{E}')$.
\end{proof}



\proputonou*

\begin{proof}
{\bf Proof of \ref{prop:utonou}.1}.
Suppose $\fml{X}\in\mbb{A}_{\mfrak{u}}(\fml{E}'_r)$, then $\fml{X}$ is a weakAXp of $\fml{E}'_0$ because $\mbf{v}\sqsubseteq\mbf{z}$.
To prove the minimality we proceed by contradiction.
Suppose by contradiction there exists an AXp $\fml{W}$ of $\fml{E}'_0$, such that, $\fml{W} \subsetneq \fml{X}$, then:
$$
  \forall(\mbf{x}\in\mbb{F}).
  \left[
  \bigland\nolimits_{i\in{\fml{W}}}(x_i=v_i)
  \right]
  \limply
  (\kappa(\mbf{x}) \in \fml{T})
$$
Then, because $\fml{W} \subsetneq \fml{X} \subset \fml{S}$:
$$
  \forall(\mbf{x}\in\mbb{F}).
  \left[
  \bigland\nolimits_{i\in{\fml{W}}}(x_i=z_i)
  \right]
  \limply
  (\kappa(\mbf{x}) \in \fml{T})
$$
that is, $\fml{W}$ is a weakAXp of $\fml{E}'_r$, which is a contradiction of minimality of the AXp $\fml{X}$ for $\fml{E}'_r$.
Thus $\mbb{A}_{\mfrak{u}}(\fml{E}'_r)\subseteq\mbb{A}_{\mfrak{u}}(\fml{E}'_0)$.

{\bf Proof of \ref{prop:utonou}.2}.
Consider $\fml{Y}\in\mbb{C}_{\mfrak{u}}(\fml{E}'_r)$, then let $\mbf{x} \in \mbb{F}$ be such that:
$$
  \left(\land_{i\in\fml{S}\setminus\fml{Y}}(x_i=z_i)\right)
  \land
  \left(\kappa(\mbf{x})\not\in\fml{T}\right)
$$
Because $\mbf{v}\sqsubseteq\mbf{z}$, then $\mbf{x}$ also satisfies:
$$
  \left(\land_{i\in\fml{F}\setminus(\fml{Y}\cup\fml{U})}(x_i=v_i)\right)
  \land
  \left(\kappa(\mbf{x})\not\in\fml{T}\right)
$$
that means, $\fml{Y}\cup\fml{U}$ is a weakCXp of $\fml{E}'_0$.
From the weakCXp $\fml{Y}\cup\fml{U}$ we can obtain CXp's contained in $\fml{Y}\cup\fml{U}$.\\
What is left to prove is that $\fml{Y} \subseteq \fml{W}$ for any CXp $\fml{W} \subseteq \fml{Y} \cup \fml{U}$.
Consider a CXp $\fml{W} \subseteq \fml{Y} \cup \fml{U}$.
Let $I_\fml{Y} = \fml{Y} \setminus \fml{W}$, and $\fml{Z} = \fml{W} \setminus \fml{Y}$.
Suppose by contradiction that $I_\fml{Y} \neq \emptyset$.
Observe that $I_\fml{Y} \not\subseteq \fml{W}$, $\fml{Z}\subseteq \fml{U}$, and that $\fml{W} = (\fml{Y} \setminus I_\fml{Y})\cup \fml{Z}$.\\
Because $\fml{W}$ is a CXp of $\fml{E}'_0$, there exists $\mbf{x} \in \mbb{F}$ such that:
$$
  \left(\land_{i\in\fml{F} \setminus \fml{W}} (x_i=v_i)\right)
  \land
  \left(\kappa(\mbf{x})\not\in\fml{T}\right)
$$
then it is also true that $\mbf{x}$ satisfies:
$$
  \left(\land_{i\in\fml{F} \setminus \left[ (\fml{Y} \setminus I_\fml{Y}) \cup \fml{U}\right]}(x_i=z_i)\right)
  \land
  \left(\kappa(\mbf{x})\not\in\fml{T}\right)
$$
that is:
$$
  \left(\land_{i\in\fml{S} \setminus (\fml{Y} \setminus I_\fml{Y})}(x_i=z_i)\right)
  \land
  \left(\kappa(\mbf{x})\not\in\fml{T}\right)
$$
that means $\fml{Y} \setminus I_\fml{Y}$ is a weakCXp for $\fml{E}'_r$, which is a contradiction of minimality of CXp $\fml{Y}$ for $\fml{E}'_r$.
Thus $\fml{Y} \subseteq \fml{W}$ for any CXp $\fml{W} \subseteq \fml{Y} \cup \fml{U}$ of $\fml{E}'_0$.
\end{proof}

\propaxpsvscxps*

\begin{proof}
{\bf Proof of \ref{prop:axps_vs_cxps}.1}.
Consider the process of computing $\fml{X}$, and the computed weak AXp’s given by the sequence
$\langle\fml{X}_r=\fml{S},\fml{X}_{r+1},\ldots,\fml{X}_{r+s}=\fml{X}\rangle$.
Between each pair $(\fml{X}_{j}, ~\fml{X}_{j+1})$ ($r \le j < {r+s}$), a feature $i_{j} (\in \fml{X}_{j})$ is tested to be left unspecified, that is, check if the predicate $\predicate_{\tn{axp}}(\fml{X}_j\setminus\{i_j\})$ is true or not.
If the predicate is false, then $\fml{X}_{(j+1)} = \fml{X}_j$, and obviously, $\mbb{A}_{\mfrak{u}}(\fml{E}'_{j+1}) = \mbb{A}_{\mfrak{u}}(\fml{E}'_j)$.

If the predicate is true, then $\fml{X}_{j+1} = \fml{X}_{j} \setminus \{i_j\}$.
Consider $\fml{W}\in\mbb{A}_{\mfrak{u}}(\fml{E}'_{j+1})$, then $\fml{W} \subseteq \fml{X}_{j+1}$ and $\axp(\fml{W}) = \top$.
As such, $\fml{W} \subseteq \fml{X}_j$, and $\fml{W}$ is subset minimal verifying $\waxp(\fml{W}) = \top$; that is $\fml{W}\in\mbb{A}_{\mfrak{u}}(\fml{E}'_j)$.

We just proved that $\mbb{A}_{\mfrak{u}}(\fml{E}'_{j+1}) \subseteq \mbb{A}_{\mfrak{u}}(\fml{E}'_j)$ for $r\le j < r+s$.
By transitivity, if $\fml{W}\in\mbb{A}_{\mfrak{u}}(\fml{E}'_k)$, then
    $\fml{W}\in\mbb{A}_{\mfrak{u}}(\fml{E}'_j)$, with $r\le j \le k \le r+s $.

{\bf Proof of \ref{prop:axps_vs_cxps}.2}.
As before consider the process of computing $\fml{X}$, and the sequence $\langle\fml{X}_r=\fml{S},\fml{X}_{r+1},\ldots,\fml{X}_{r+s}=\fml{X}\rangle$.
Between each pair $(\fml{X}_{j}, ~\fml{X}_{j+1})$ ($r \le j < {r+s}$), a feature $i_{j} (\in \fml{X}_{j})$ is tested to be left unspecified, that is, check if the predicate $\predicate_{\tn{axp}}(\fml{X}_j\setminus\{i_j\})$ is true or not.
If the predicate is false, then $\fml{X}_{j+1} = \fml{X}_j$, and obviously, $\mbb{C}_{\mfrak{u}}(\fml{E}'_{j+1}) = \mbb{C}_{\mfrak{u}}(\fml{E}'_j)$.
Let $\fml{Z}_j = \emptyset (\subseteq \fml{X}_j)$, then if $\fml{Y} \in \mbb{C}_{\mfrak{u}}(\fml{E}'_{j+1})$ then $\fml{Y} \cup \fml{Z}_j \in \mbb{C}_{\mfrak{u}}(\fml{E}'_j)$.

If the predicate is true, then $\fml{X}_{j+1} = \fml{X}_j \setminus \{i_j\}$.
Consider $\fml{Y} \in \mbb{C}_{\mfrak{u}}(\fml{E}'_{j+1})$, then $\fml{Y}\subseteq \fml{X}_{j+1}$ is a subset minimal set satisfying $\wcxp(\fml{Y}) = \top$, that is:
$$
  \exists(\mbf{x}\in\mbb{F}).
  \left(\land_{i\in\fml{X}_{j+1}\setminus\fml{Y}}(x_i=z_i)\right)
  \land
  \left(\kappa(\mbf{x})\not\in\fml{T}\right)
$$

There are two cases to consider, depending on the necessity of including feature $i_j$ in the CXp for the set of fixed features $\fml{X}_j$.
In the first case, feature $i_j$ is required in the CXp.
This can happen because $i_j$ is not a part of $\fml{X}_{j+1}$, as such in this case, it happens that:
$$
  \forall(\mbf{x}\in\mbb{F}).
  \left(\land_{i\in\fml{X}_{j}\setminus\fml{Y}}(x_i=z_i)\right)
  \limply
  \left(\kappa(\mbf{x})\in\fml{T}\right)
$$
In this case, let $\fml{Z}_j = \{i_j\}(\subseteq \fml{X}_j)$, then we can consider $\fml{Y} \cup \fml{Z}_j$, and because $\wcxp(\fml{Y}) = \top$, we know that:
$$
  \exists(\mbf{x}\in\mbb{F}).
  \left(\land_{i\in\fml{X}_{j} \setminus (\fml{Y} \cup \fml{Z}_j)}(x_i=z_i)\right)
  \land
  \left(\kappa(\mbf{x})\not\in\fml{T}\right)
$$
The subset minimality of $\fml{Y} \cup \fml{Z}_j$ is due to the subset minimality of $\fml{Y}$ in $\fml{X}_{j+1}$.

In the second case, feature $i_j$ is not required in the CXp, that is, it is enough to consider $\fml{Y}$ in the set $\fml{X}_j$ so that:
$$
  \exists(\mbf{x}\in\mbb{F}).
  \left(\land_{i\in\fml{X}_{j} \setminus \fml{Y} }(x_i=z_i)\right)
  \land
  \left(\kappa(\mbf{x})\not\in\fml{T}\right)
$$
Let $\fml{Z}_j = \emptyset (\subseteq \fml{X}_j)$, then if $\fml{Y} \in \mbb{C}_{\mfrak{u}}(\fml{E}'_{j+1})$ then $\fml{Y} \cup \fml{Z}_j \in \mbb{C}_{\mfrak{u}}(\fml{E}'_j)$.

We have proven that for each pair $(\fml{X}_{j}, ~\fml{X}_{j+1})$ ($r \le j < {r+s}$), if $\fml{Y} \in \mbb{C}_{\mfrak{u}}(\fml{E}'_{j+1})$, then there exists $\fml{Z}_j\subseteq \fml{X}_{j}$ such that $\fml{Y} \cup \fml{Z}_j \in \mbb{C}_{\mfrak{u}}(\fml{E}'_j)$.

Consider now $j$, $k$ such that $r\le j\le k\le r+s$, and $\fml{Y} \in \mbb{C}_{\mfrak{u}}(\fml{E}'_{k})$.
Consider the sequence of pairs
$\langle (\fml{X}_{k-1}, \fml{X}_k),
\ldots,
(\fml{X}_j, \fml{X}_{j+1})\rangle$.
Starting from the first pair, apply the previous proven result and obtain that there exists $\fml{Z}_{k-1}\subseteq \fml{X}_{k-1}$ such that $\fml{Y} \cup \fml{Z}_{k-1} \in \mbb{C}_{\mfrak{u}}(\fml{E}'_{k-1})$.
For the next pair use the new CXp $\fml{Y} \cup \fml{Z}_{k-1}$ and obtain that there exists $\fml{Z}_{k-2}\subseteq \fml{X}_{k-2}$ such that $\fml{Y} \cup \fml{Z}_{k-1} \cup \fml{Z}_{k-2} \in \mbb{C}_{\mfrak{u}}(\fml{E}'_{k-2})$.
Proceed in a similar way until the last pair to obtain that there exists $\fml{Z}_{j}\subseteq \fml{X}_{j}$ such that $\fml{Y} \cup \fml{Z}_{k-1}\cup \ldots \fml{Z}_{j} \in \mbb{C}_{\mfrak{u}}(\fml{E}'_j)$.
Since $\fml{X}_{k-1}\subseteq \ldots \subseteq \fml{X}_{j}$, and $\fml{Z}_{k-1}\cup \ldots \fml{Z}_{j} \subseteq \fml{X}_{j}$, then we can consider $\fml{Z} = \bigcup_{t=j}^{k-1} \fml{Z}_{t}$, $\fml{Z} \subseteq \fml{X}_{j}$, and $\fml{Y} \cup \fml{Z} \in \mbb{C}_{\mfrak{u}}(\fml{E}'_j)$ as required.
\end{proof}

\propndual*

\begin{proof}
Consider the process of computing $\fml{X}$, and the computed weak AXp’s given by the sequence
$\langle\fml{X}_r=\fml{S},\fml{X}_{r+1},\ldots,\fml{X}_{r+s}=\fml{X}\rangle$.
Between each pair $(\fml{X}_{j}, ~\fml{X}_{j+1})$ ($r \le j < {r+s}$), a feature $i_{j} (\in \fml{X}_{j})$ is tested to be left unspecified, that is, check if the predicate $\predicate_{\tn{axp}}(\fml{X}_j\setminus\{i_j\})$ is true or not.
If the predicate is true then $\fml{X}_{j+1} = \fml{X}_j \setminus \{i_j\}$, otherwise $\fml{X}_{j+1} = \fml{X}_j$.

Consider $j$ such that $r \le j \le r+s$.
The set $\fml{X}_j$ divide $\fml{F}$ in two, the set of fixed features $\fml{S}=\fml{X}_j$ and the set of unspecified features $\fml{U}=\fml{F} \setminus \fml{X}_j$.
Let $\mbf{z}'$ be such that $z'_i = z_i$ if $i \in \fml{X}_j$, undefined otherwise.
Then, $\fml{E}'_j = (\fml{C},(\mbf{z}',\fml{T}))$ is a generalized explanation problem.
Due to Proposition~\ref{prop:dualmf}, then each AXp in $\mbb{A}_{\mfrak{u}}(\fml{E}'_j)$ is a minimal
  hitting set of the CXp's in $\mbb{C}_{\mfrak{u}}(\fml{E}'_j)$, and
  vice-versa.
\end{proof}
%

\section{Mapping of Class Names to Numbers}

\cref{tab:map:derm,tab:map:soybean,tab:map:zoo} show the mapping
between classes names and features used in the paper. For the
remaining datasets, i.e.\ \texttt{auto} and \texttt{student-por}, the
classes are already numeric.

\begin{table}[ht]
  \begin{center}
    \begin{tabular}{lc}
      \toprule
      Name & Number \\
      \toprule
      psoriasis & 1\\
      seboreic dermatitis & 2\\
      lichen planus & 3 \\
      pityriasis rosea & 4 \\
      cronic dermatitis & 5 \\
      pityriasis rubra pilaris & 6 \\
      \bottomrule
    \end{tabular}
    \caption{Class mapping for \texttt{dermatology}} \label{tab:map:derm}
  \end{center}
\end{table}

\begin{table}[ht]
  \begin{center}
    \begin{tabular}{lc}
      \toprule
      Name & Number \\
      \toprule
      diaporthe-stem-canker & 1 \\
      charcoal-rot & 2 \\
      rhizoctonia-root-rot & 3 \\
      phytophthora-rot & 4 \\
      brown-stem-rot & 5 \\ 
      powdery-mildew & 6 \\
      downy-mildew & 7 \\ 
      brown-spot & 8 \\ 
      bacterial-blight & 9 \\
      bacterial-pustule & 10 \\ 
      purple-seed-stain & 11 \\ 
      anthracnose & 12 \\
      phyllosticta-leaf-spot & 13 \\ 
      alternarialeaf-spot & 14 \\
      frog-eye-leaf-spot & 15 \\ 
      diaporthe-pod-and-stem-blight & 16 \\
      cyst-nematode & 17 \\ 
      2-4-d-injury & 18 \\ 
      herbicide-injury & 19 \\
      \bottomrule
    \end{tabular}
    \caption{Class mapping for \texttt{soybean}} \label{tab:map:soybean}
  \end{center}
\end{table}

\begin{table}[ht]
  \begin{center}
    \begin{tabular}{lc}
      \toprule
      Name & Number \\
      \toprule
      mammal & 1 \\
      reptile & 2 \\
      bug & 3 \\
      bird & 4 \\
      invertebrate & 5 \\
      amphibian & 6 \\
      fish & 7 \\
      \bottomrule
    \end{tabular}
    \caption{Class mapping for \texttt{zoo}} \label{tab:map:zoo}
  \end{center}
\end{table}


\section{Summary of Notation}

Below we summarize the notation used in some of the paper's sections.

\input{definitionsInEquations}
\section{Medical Diagnosis Case Study -- Feature description}
\input{meddiag}


\section{Illustrative Examples}

Below we detail examples of traces for the algorithms proposed in 
the paper using the running \cref{ex:runex01}.

\subsection{Computation of One AXp/CXp}

\input{./examples/onexp}

\subsection{Enumeration of AXp's/CXp's}

\input{./examples/allxp}

%% file: definitionsInEquations.tex
\subsection*{Used in~\cref{sec:prelim}}

$\fml{F}=\{1,\ldots,m\}$, set of features \\
$\fml{K}=\{c_1,c_2,\ldots,c_K\}$, set of classes\\
$\mbb{D}=\langle\mbb{D}_1,\ldots,\mbb{D}_m\rangle$, domains of features \\
$\mbb{F}=\mbb{D}_1\times{\mbb{D}_2}\times\ldots\times{\mbb{D}_m}$, feature space \\
$\mbf{x}=(x_1,\ldots,x_m)$, arbitrary point in $\mbb{F}$ \\
$X=\{x_1,\ldots,x_m\}$, set of variables \\
$\mbf{v}=(v_1,\ldots,v_m)$, specific point in $\mbb{F}$ \\
$\fml{M}$, classifier \\
$\kappa:\mbb{F}\to\fml{K}$, classification function of $\fml{M}$ \\
$\fml{C}=(\fml{F},\mbb{D},\mbb{F},\fml{K},\kappa)$, classification problem \\
$(\mbf{v}, c)$, instance \\
$\fml{E}=(\fml{C},(\mbf{v},c))$, explanation problem \\
$
  \waxp(\fml{X}) := \forall(\mbf{x}\in\mbb{F}).
  \left[
  \bigland\nolimits_{i\in{\fml{X}}}(x_i=v_i)
  \right]
  \limply(\kappa(\mbf{x})=c)
$,
from \eqref{eq:axp} \\
%
%
$
  \wcxp(\fml{Y}) :=
  \exists(\mbf{x}\in\mbb{F}).
  \left[
  \bigland\nolimits_{i\not\in{\fml{Y}}}(x_i=v_i)
  \right]
  \land(\kappa(\mbf{x})\not=c)
$,
from \eqref{eq:cxp} \\
%
%
$
  \axp(\fml{X}) :=
  \waxp(\fml{X})
  \land_{t\in\fml{X}}\neg\waxp(\fml{X}\setminus\{t\})
$,
from \eqref{eq:axp3} \\
$
  \cxp(\fml{Y}) :=
  \wcxp(\fml{Y})
  \land_{t\in\fml{Y}}\neg\wcxp(\fml{Y}\setminus\{t\})
$,
from \eqref{eq:cxp3} \\
$
  \mbb{A}(\fml{E}) = \{\fml{X}\subseteq\fml{F}\,|\,\axp(X)\}
$ \\
$
  \mbb{C}(\fml{E}) = \{\fml{Y}\subseteq\fml{F}\,|\,\cxp(Y)\}
$ \\
%

\subsection*{Used in~\cref{sec:xpsi}}

$\mfrak{u}$, distinguished value not in any $\mbb{D}_i$ \\
$\mbb{D}'_i=\mbb{D}_i\cup\{\mfrak{u}\}$, extended domain\\
$
\mbb{F}'=\mbb{D}'_1\times\cdots\times\mbb{D}'_m
$,
extended feature space \\
$\fml{U}$, set of features with unspecified values \\
$\fml{S}$, set of features with specified values \\
$
\mbf{v}\sqsubseteq\mbf{z} :=
\forall(i\in\fml{F}).\left[(v_i=z_i)\lor({z_i}=\mfrak{u})\right]
$,
$\mbf{v}$ covered by $\mbf{z}$ \\
$\kappa':\mbb{F}'\to2^{\fml{K}}$, generalization of classification function \\
$
\kappa'(\mbf{z})=\{c\in\fml{K}\,|\,\mbf{v}\in\mbb{F}\land\mbf{v}\sqsubseteq\mbf{z}\land{c}=\kappa(\mbf{v})\}
$,
from \eqref{eq:defkp} \\
$\fml{T}\subseteq\fml{K}$, target classes\\
$\fml{E}'=(\fml{C},(\mbf{z},\fml{T}))$, generalized explanation problem \\
%
%
$
  \forall(\mbf{x}\in\mbb{F}) . \left(\land_{i\in\fml{S}}(x_i=z_i)\right)\limply\left(\kappa(\mbf{x})\in\fml{T}\right)
$,
$\mbf{z}$ sufficient for $\fml{T}$,
from \eqref{eq:suffpred} \\
$
  \waxp(\fml{X}) :=
   \forall(\mbf{x}\in\mbb{F}).
    \left(\land_{i\in\fml{X}}(x_i=z_i)\right)
    \limply
    \left(\kappa(\mbf{x})\in\fml{T}\right)
$ \\
AXp minimal set $\fml{X}\subseteq\fml{S}$ s.t. $\waxp(\fml{X})=\top$ \\
$
  \wcxp(\fml{Y}) :=
  \exists(\mbf{x}\in\mbb{F}).
  \left(\land_{i\in\fml{S}\setminus\fml{Y}}(x_i=z_i)\right)
  \land
  \left(\kappa(\mbf{x})\not\in\fml{T}\right)
$ \\
CXp minimal set $\fml{Y}\subseteq\fml{S}$ s.t. $\wcxp(\fml{Y})=\top$ \\
$
\mbb{A}_{\mfrak{u}}(\fml{E}') = \{\fml{X}\subseteq\fml{S}\,|\,\axp(X)\}
$ \\
$
\mbb{C}_{\mfrak{u}}(\fml{E}') = \{\fml{Y}\subseteq\fml{S}\,|\,\cxp(Y)\}
$ \\
%
%
\begin{align}
  \predicate_{\tn{axp}}(\fml{W};&\mbb{T},\fml{F},\fml{S},\kappa,\mbf{z},\fml{T})
  \triangleq \nonumber \\
  \neg\consistent&\left(\left\llbracket\left(\bigwedge\limits_{i\in\fml{W}}(x_i=z_i)\right)\land(\kappa(\mbf{x})\not\in\fml{T})\right\rrbracket\right) \nonumber
\end{align}
from \eqref{eq:predaxp} \\
\begin{align}
  \predicate_{\tn{cxp}}(\fml{W};&\mbb{T},\fml{F},\fml{S},\kappa,\mbf{z},\fml{T})
  \triangleq \nonumber\\
  \consistent&\left(\left\llbracket\left(\bigwedge\limits_{i\in\fml{S}\setminus\fml{W}}(x_i=z_i)\right)\land(\kappa(\mbf{x})\not\in\fml{T})\right\rrbracket\right) \nonumber
\end{align}
from \eqref{eq:predcxp} \\

%% file: meddiag.tex

\input{./texfigs/casestudy}


\cref{fig:cstudy} shows a decision tree for the non-invasive diagnosis
of a concrete type of meningitis, Meningococcal Disease (MD), obtained
from~\cite[Figure~9]{belmonte-ieee-access20}. The DT is to be used on
suspected cases of meningitis, and the goal is to decide whether the
type is MD or not.
As can be observed from the names of the features, some refer to
traits of a patient (e.g.\ age, gender, and place where she lives), and
some other refer to symptoms that the patient exhibits (e.g.\ being in
a coma, headache, seizures, vomiting, stiff neck, and existence of
petechiae).

We want to assess the quality of the classifier, and so we seek to
understand whether a patient can be diagnosed with MD meningitis
without exhibiting any of the symptoms.
The proposed procedure operates as follows. For each leaf predicting
the class of interest (i.e.\ a diagnosis of MD meningitis) we identify
the features representing symptoms and, among these, we pick the ones
that are active (i.e.\ the patient exhibits the symptom).
We then declare the other features (non-active or non-symptom) as
specified, and leave the active symptom features as unspecified.
%
If the specified features are sufficient for the prediction, then we
know that
a patient can be diagnosed with MD meningitis despite not exhibiting
any symptoms. Notwithstanding, we proceed to compute an AXp, since
this offers an irreducible set of features sufficient for the unwanted
prediction. (Although we consider a concrete use case, the procedure
outlined above is completely general as an arbitrary number of
identified features could be considered.)

For the example DT, and for the path shown, the only feature that is a
symptom and which is active is $V$. We then search for an AXp where
this feature is left unspecified. As can be concluded, there exists
such an AXp.
Thus, the DT can serve to diagnose MD meningitis for patients without
any active (i.e.\ true) symptom, among the symptoms considered. The
active symptom feature in the path is only misleadingly being used for
the prediction; as the computed AXp confirms, the active feature is
redundant, and so irrelevant for the prediction.
The ensuing conclusion is that the DT proposed
in~\cite{belmonte-ieee-access20} exhibits unrealistic diagnoses. As
demonstrated by the computed AXp, any human being suspect of having
meningitis, older than 5 years of age, that is a male, with an urban
residence (i.e.\ $\tn{Zone}{=}2$), and exhibiting no headache will be
diagnosed as having MD meningitis!
By applying a similar analysis on path $\langle1,3,6,8,11\rangle$, and
by leaving again $V$ unspecified, we also conclude that any patient
older than 5 years of age that is not Vomiting will also be diagnosed
with MD meningitis.
Finally, an even stronger test can be made with respect to path
$\langle1,3,6,8,10,14\rangle$. In this case, we can declare \emph{any}
of the symptoms unspecified, and conclude that $A$ and $Z$ is
sufficient for the prediction. The resulting AXp is $\{A,Z\}$. Hence,
the prediction of MD meningitis involves testing no symptoms
whatsoever.

The identified problems can result from a very unbalanced dataset, an
incorrect tree inducer, or some critical features that are not being
accounted for. It seems unlikely that the issues with the DT might
result from a flipped test, e.g.\ the test on $V$, since flipping the
values of the test on $V$ would yield the same conclusions.
As can be observed, the issue with the DT of~\cref{fig:cstudy} is not
easily spotted by inspection, concretely for the longer path.
Furthermore, such analysis by inspection would be far more challenging
on larger DTs. More importantly, the proposed analysis is made
feasible by the generalized explainability solution proposed in this paper.


%% file: texfigs/casestudy.tex
\begin{figure*}[t]
  \captionsetup[subfigure]{aboveskip=-1pt,belowskip=-1pt}
  \begin{minipage}{0.375\textwidth}
    \begin{subfigure}{\textwidth}
      \scalebox{0.925}{\input{./texfigs/tree02}}
    \end{subfigure}

  \end{minipage}
  \begin{minipage}{0.625\textwidth}
    \begin{subfigure}{\textwidth}
      \begin{center}
        \renewcommand{\arraystretch}{1.0125}
        \begin{tabular}{cccccc} \toprule
          Feat.~\# & Name & Meaning & Definition & Domain & Trait/Symp. \\ \toprule 
          1 & $A$ & Age & $\tn{Age}>5$? & $\{0,1\}$ & T \\ \midrule
          2 & $P$ & Petechiae & Petechiae? & $\{0,1\}$ & S \\ \midrule 
          3 & $N$ & Stiff Neck & Stiff Neck? & $\{0,1\}$ & S \\ \midrule 
          4 & $V$ & Vomiting & Vomiting? & $\{0,1\}$ & S \\ \midrule 
          5 & $Z$ & Zone & Zone${=}$? & $\{0,1,2\}$ & T\\ \midrule 
          6 & $S$ & Seizures & Seizures? & $\{0,1\}$ & S \\ \midrule
          7 & $G$ & Gender & Gender? & $\{0\,\tn{(F)},1\,\tn{(M)}\}$ & T \\ \midrule
          8 & $H$ & Headache & Headache? & $\{0,1\}$ & S \\ \midrule
          9 & $C$ & Coma & Coma? & $\{0,1\}$ & S \\
          \bottomrule
        \end{tabular}
      \end{center}
    \end{subfigure}

    \bigskip\smallskip

    \begin{subfigure}{\textwidth}
      \begin{center}
        \renewcommand{\arraystretch}{1.0125}
        \begin{tabular}{cc} \toprule
          Path \& Prediction &
          $\langle1,3,6,8,10,13,16,20,22,23\rangle$ \& \tbf{Y} \\ \midrule
          Features in path & $\{A,P,N,V,Z,S,H,C,G\}$ \\ \midrule
          Symptoms in path & $\{P,N,V,S,H,C\}$ \\ \midrule
          Active symptoms & $\{V\}$ \\ \midrule 
          Computed AXp & $\{A,G,H,Z\}$ \\ \midrule
          Interpretation &
          $\tn{Age}{\ge}5\land\tn{Male}\land\tn{no~Headache}\land\tn{Zone}{=}\tn{2}$
          \\ \toprule 
          Conclusion & MD meningitis can be predicted without symptoms\\
          \bottomrule
        \end{tabular}
      \end{center}
    \end{subfigure}
  \end{minipage}
  \caption{DT for non-invasive diagnosis of MD meningitis.} \label{fig:cstudy}
\end{figure*}

%% file: texfigs/tree02.tex
%
\forestset{
  BDT/.style={
    for tree={
      l=1.5cm,s sep=1.0cm,
      if n children=0{}{circle}, 
      draw=black,
      text=black,
      edge={-{Stealth[]}},
    }
  },
}
\begin{forest}
  BDT
  [{$A$}, label={[yshift=-6.5ex]{{\xtiny1}}} 
    [{$P$}, label={[yshift=-6.5ex]{{\xtiny2}}}, 
      edge label={node[midway,left,xshift=-0.5pt] {{\xscriptsize$=0$}}}
          [\rhlight{\tbf{Y}}, label={[yshift=-5.125ex]{{\xtiny4}}},
            edge label={node[midway,left,xshift=-0.5pt] {{\xscriptsize$=1$}}},
            rectangle, fill={tred3!20}
          ]
          [\dghlight{\tbf{N}}, label={[yshift=-5.125ex]{{\xtiny5}}},
            edge label={node[midway,right,xshift=0.5pt] {{\xscriptsize$=0$}}},
            rectangle, fill={tgreen3!25}
          ]
    ]
    [{$P$}, label={[yshift=-6.5ex]{{\xtiny3}}}, 
      edge label={node[midway,right,xshift=0.5pt] {{\xscriptsize$=1$}}}
      [{$N$}, label={[yshift=-6.725ex]{{\xtiny6}}}, 
        edge label={node[midway,left,xshift=-2.25pt] {{\xscriptsize$=0$}}}
        [{$V$}, label={[yshift=-6.5ex]{{\xtiny8}}}, 
          edge label={node[midway,left,xshift=-1.75pt] {{\xscriptsize$=0$}}}
          [{$Z$}, label={[xshift=-3.35ex,yshift=-3.5ex]{{\xtiny10}}}, 
            edge label={node[midway,left,xshift=-2.0pt] {{\xscriptsize$=1$}}}
            [\dghlight{\tbf{N}}, label={[yshift=-5.125ex]{{\xtiny12}}},
              edge label={node[midway,left,xshift=-0.5pt] {{\xscriptsize$=1$}}},
              rectangle, fill={tgreen3!25}
            ]
            [{$S$}, label={[yshift=-6.5ex]{{\xtiny13}}}, 
              edge label={node[near end,right,xshift=0.5pt] {{\xscriptsize$=2$}}}
              [{$G$}, label={[yshift=-6.5ex]{{\xtiny15}}}, 
                edge label={node[midway,left,xshift=-2.5pt] {{\xscriptsize$=1$}}}
                [\dghlight{\tbf{N}}, label={[yshift=-5.125ex]{{\xtiny17}}},
                  edge label={node[midway,left,xshift=0.5pt] {{\xscriptsize$=0$}}},
                  rectangle, fill={tgreen3!25}
                ]
                [\rhlight{\tbf{Y}}, label={[yshift=-5.125ex]{{\xtiny18}}},
                  edge label={node[midway,right,xshift=0.5pt] {{\xscriptsize$=1$}}},
                  rectangle, fill={tred3!20}
                ]
              ]
              [{$H$}, label={[yshift=-6.575ex]{{\xtiny16}}}, 
                edge label={node[midway,right,xshift=0.5pt] {{\xscriptsize$=0$}}}
                [\dghlight{\tbf{N}}, label={[yshift=-5.125ex]{{\xtiny19}}},
                  edge label={node[midway,left,xshift=0.5pt] {{\xscriptsize$=1$}}},
                  rectangle, fill={tgreen3!25}
                ]
                [{$C$}, label={[yshift=-6.5ex]{{\xtiny20}}}, 
                  edge label={node[midway,right,xshift=0.5pt] {{\xscriptsize$=0$}}}
                  [\rhlight{\tbf{Y}}, label={[yshift=-5.125ex]{{\xtiny21}}},
                    edge label={node[midway,left,xshift=-0.5pt] {{\xscriptsize$=1$}}},
                    rectangle, fill={tred3!20}
                  ]
                  [{$G$}, label={[yshift=-6.5ex]{{\xtiny22}}}, 
                    edge label={node[midway,right,xshift=0.5pt] {{\xscriptsize$=0$}}}
                    [\rhlight{\tbf{Y}}, label={[yshift=-5.125ex]{{\xtiny23}}},
                      edge label={node[midway,left,xshift=-0.5pt] {{\xscriptsize$=1$}}},
                      rectangle, fill={tred3!20}
                    ]
                    [\dghlight{\tbf{N}}, label={[yshift=-5.125ex]{{\xtiny24}}},
                      edge label={node[midway,right,xshift=0.5pt] {{\xscriptsize$=0$}}},
                      rectangle, fill={tgreen3!25}
                    ]
                  ]
                ]
              ]
            ]
            [\rhlight{\tbf{Y}}, label={[yshift=-5.125ex]{{\xtiny14}}},
              edge label={node[midway,right,xshift=0.5pt] {{\xscriptsize$=0$}}},
              rectangle, fill={tred3!20}
            ]
          ]
          [\rhlight{\tbf{Y}}, label={[yshift=-5.125ex]{{\xtiny11}}},
            edge label={node[midway,right,xshift=0.5pt] {{\xscriptsize$=0$}}},
            rectangle, fill={tred3!20}
          ]
        ]
        [\rhlight{\tbf{Y}}, label={[yshift=-5.125ex]{{\xtiny9}}},
          edge label={node[midway,right,xshift=0.5pt] {{\xscriptsize$=1$}}},
          rectangle, fill={tred3!20}
        ]
      ]
      [\rhlight{\tbf{Y}}, label={[yshift=-5.125ex]{{\xtiny7}}},
        edge label={node[midway,right,xshift=0.5pt] {{\xscriptsize$=1$}}},
        rectangle, fill={tred3!20}
      ]
    ]
  ]
\end{forest}

%% file: examples/onexp.tex
\begin{table*}[th]
  \input{examples/algorithm1-trace1.tex}
  \caption{Computing one AXP using an ascending order of features to test.}
  \label{tab:alg1:trace1}
\end{table*}

\begin{table*}[th]
    \input{examples/algorithm1-trace2.tex}
    \caption{Computing one AXP using a descending order of features to test.}
    \label{tab:alg1:trace2}
\end{table*}

Consider $\mbf{z}=(35.0,0.25,\mfrak{u},100.0,\mfrak{u},5.0)$, with
$\kappa'(\mbf{z})=\{\oplus,\oplus\oplus\}$.
\cref{tab:alg1:trace1,tab:alg1:trace2} show possible traces
of~\cref{alg:onexp} to compute one AXp using an ascending, and
descending order of features to test respectively.

\begin{table*}[th]
    \input{examples/algorithm1-trace3.tex}
    \caption{Computing one CXP using an ascending order of features to test.}
    \label{tab:alg1:trace3}
\end{table*}

\begin{table*}[th]
    \input{examples/algorithm1-trace4.tex}
    \caption{Computing one CXP using a descending order of features to test.}
    \label{tab:alg1:trace4}
\end{table*}

\cref{tab:alg1:trace3,tab:alg1:trace4} show possible traces
of~\cref{alg:onexp} traces to compute one CXp using an ascending, and
descending order of features to test respectively.

%% file: examples/algorithm1-trace1.tex
\begin{center}
  \begin{tabular}{ccc}
    \toprule
    $i$ & $\mbb{P}_{\mrm{axp}}(\fml{W}\setminus \{i\})$ & $\fml{W}$ \\
    \toprule
    -- & -- & $\{1,2,4,6\}$ \\
    1  & \cmark & $\{2,4,6\}$ \\
    2  & \cmark & $\{4,6\}$ \\
    4  & \xmark & $\{4,6\}$ \\
    6  & \xmark & $\{4,6\}$ \\
    \bottomrule
  \end{tabular}
\end{center}
%

%% file: examples/algorithm1-trace2.tex
\begin{center}
  \begin{tabular}{ccc}
    \toprule
    $i$ & $\mbb{P}_{\mrm{axp}}(\fml{W}\setminus \{i\})$ & $\fml{W}$ \\
    \toprule
    -- & -- & $\{1,2,4,6\}$ \\
    6  & \cmark & $\{1,2,4\}$ \\
    4  & \xmark & $\{1,2,4\}$ \\
    2  & \cmark & $\{1,4\}$ \\
    1  & \xmark & $\{1,4\}$ \\
    \bottomrule
  \end{tabular}
\end{center}
%

%% file: examples/algorithm1-trace3.tex
\begin{center}
  \begin{tabular}{ccc}
    \toprule
    $i$ & $\mbb{P}_{\mrm{cxp}}(\fml{W}\setminus \{i\})$ & $\fml{W}$ \\
    \toprule
    -- & -- & $\{1,2,4,6\}$ \\
    1  & \cmark & $\{2,4,6\}$ \\
    2  & \cmark & $\{4,6\}$ \\
    4  & \xmark & $\{4,6\}$ \\
    6  & \cmark & $\{4\}$ \\
    \bottomrule
  \end{tabular}
\end{center}
%

%% file: examples/algorithm1-trace4.tex
\begin{center}
  \begin{tabular}{ccc}
    \toprule
    $i$ & $\mbb{P}_{\mrm{cxp}}(\fml{W}\setminus \{i\})$ & $\fml{W}$ \\
    \toprule
    -- & -- & $\{1,2,4,6\}$ \\
    6  & \cmark & $\{1,2,4\}$ \\
    4  & \xmark & $\{1,2,4\}$ \\
    2  & \cmark & $\{1,4\}$ \\
    1  & \cmark & $\{4\}$ \\
    \bottomrule
  \end{tabular}
\end{center}
%

%% file: examples/allxp.tex
Consider $\mbf{z}=(35.0,0.25,\mfrak{u},100.0,\mfrak{u},5.0)$, with
$\kappa'(\mbf{z})=\{\oplus,\oplus\oplus\}$.
%
%
\cref{tab:alg2:trace_pos,tab:alg2:trace_neg} show possible traces
of~\cref{alg:allxp} to enumerate all AXp's and CXp's.
The traces in \cref{tab:alg2:trace_pos,tab:alg2:trace_neg} were
obtained using a SAT solver that is biased to assign
variables that are free (unconstrained) in the input CNF formula,
the values of true ($1$) and false ($0$), respectively.

\begin{table*}[th]
  \input{examples/algorithm2-trace1.tex}
  \caption{Enumerating all AXp's and CXp's using a SAT solver biased
    to positive assignments.} 
  \label{tab:alg2:trace_pos}
\end{table*}

\begin{table*}[th]
  \input{examples/algorithm2-trace2.tex}
  \caption{Enumerating all AXp's and CXp's using a SAT solver biased
    to negative assignments.}
  \label{tab:alg2:trace_neg}
\end{table*}

%% file: examples/algorithm2-trace1.tex
\begin{center}
  \begin{tabular}{ccccccc}
    \toprule
    $\fml{H}$ & $(\msf{outc},\mbf{u})$ & $\fml{R}$ & $\fml{Q}$ &
    $\mbb{P}_{\mrm{cxp}}(\fml{Q})$ & $\fml{P}$ & New Clause \\
    \toprule
    $\emptyset$ & $(\top,\{1,1,1,1,1,1\})$ & $\emptyset$ &
    $\{1,2,4,6\}$ & \cmark & $\{4\}$ & $c_1=(\neg{u_4})$ \\
    $\{c_1\}$ & $(\top,\{1,1,1,0,1,1\})$ & $\{4\}$ &
    $\{1,2,6\}$ & \cmark & $\{1,6\}$ & $c_2=(\neg{u_1}\lor\neg{u_6})$ \\ 
    $\{c_1,c_2\}$ & $(\top,\{0,1,1,0,1,1\})$ & $\{1,4\}$ &
    $\{2,6\}$ & \xmark & $\{1,4\}$ & $c_3=(u_1\lor{u_4})$ \\
    $\{c_1,c_2,c_3\}$ & $(\top,\{1,1,1,0,1,0\})$ & $\{4,6\}$ &
    $\{1,2\}$ & \xmark & $\{4,6\}$ & $c_4=(u_4\lor{u_6})$\\
    $\{c_1,c_2,c_3,c_4\}$ & $(\bot,\tn{--})$ & -- & -- & -- & -- & -- \\
    \bottomrule
  \end{tabular}
\end{center}
%

%% file: examples/algorithm2-trace2.tex
\begin{center}
  \begin{tabular}{ccccccc}
    \toprule
    $\fml{H}$ & $(\msf{outc},\mbf{u})$ & $\fml{R}$ & $\fml{Q}$ &
    $\mbb{P}_{\mrm{cxp}}(\fml{Q})$ & $\fml{P}$ & New Clause \\
    \toprule
    $\emptyset$ & $(\top,\{0,0,0,0,0,0\})$ & $\{1,2,4,6\}$ &
    $\emptyset$ & \xmark & $\{4,6\}$ & $c_1=(u_4\lor{u_6})$ \\
    $\{c_1\}$ & $(\top,\{0,0,0,1,0,0\})$ & $\{1,2,6\}$ &
    $\{4\}$ & \cmark & $\{4\}$  & $c_2=(\neg{u_4})$ \\
    $\{c_1,c_2\}$ & $(\top,\{0,0,0,0,0,1\})$ & $\{1,2,4\}$ &
    $\{6\}$ & \xmark & $\{1,4\}$ & $c_3=(u_1\lor{u_4})$ \\
    $\{c_1,c_2,c_3\}$ & $(\top,\{1,0,0,0,0,1\})$ & $\{2,4\}$ &
    $\{1,6\}$ & \cmark & $\{1,6\}$ & $c_4=(\neg{u_4}\lor\neg{u_6})$ \\
    $\{c_1,c_2,c_3,c_4\}$ & $(\bot,\tn{--})$ & -- & -- & -- & -- & -- \\
    \bottomrule
  \end{tabular}
\end{center}
%
